\documentclass{article}

\usepackage{microtype}
\usepackage{graphicx}
\usepackage{subfigure}
\usepackage{booktabs} 

\usepackage{hyperref}

\usepackage[accepted]{icml2024}

\usepackage{amsmath}
\usepackage{amssymb}
\usepackage{mathtools}
\usepackage{amsthm}

\usepackage[capitalize,noabbrev]{cleveref}

\theoremstyle{plain}
\newtheorem{theorem}{Theorem}[section]
\newtheorem{proposition}[theorem]{Proposition}
\newtheorem{lemma}[theorem]{Lemma}

\theoremstyle{definition}
\newtheorem{definition}[theorem]{Definition}

\theoremstyle{remark}
\newtheorem{remark}[theorem]{Remark}

\usepackage[textsize=tiny]{todonotes}

\usepackage{url}
\usepackage{physics}
\usepackage{array}
\usepackage{ytableau}

\usepackage{thmtools}
\usepackage{thm-restate}
\theoremstyle{plain}

\usepackage{amsmath,amsfonts,bm}

\def\eqref#1{equation~\ref{#1}}

\def\1{\bm{1}}

\DeclareMathAlphabet{\mathsfit}{\encodingdefault}{\sfdefault}{m}{sl}
\SetMathAlphabet{\mathsfit}{bold}{\encodingdefault}{\sfdefault}{bx}{n}

\newcommand{\pdata}{p_{\rm{data}}}
\newcommand{\ptrain}{\hat{p}_{\rm{data}}}

\newcommand{\E}{\mathbb{E}}

\DeclareMathOperator{\sign}{sign}

\definecolor{orange}{RGB}{222, 102, 13}

\theoremstyle{plain}
\newtheorem{claim}{Claim}
\newcommand{\nvoc}{N_{\rm voc}}

\icmltitlerunning{Towards Understanding Inductive Bias in Transformers: A View From Infinity}

\begin{document}

\twocolumn[
\icmltitle{Towards Understanding Inductive Bias in Transformers: A View From Infinity}




\begin{icmlauthorlist}
\icmlauthor{Itay Lavie}{huji}
\icmlauthor{Guy Gur-Ari}{guy}
\icmlauthor{Zohar Ringel}{huji}
\end{icmlauthorlist}

\icmlaffiliation{huji}{Racah Institute of Physics, Hebrew University of Jerusalem, Jerusalem 91904, Israel}
\icmlaffiliation{guy}{Augment Computing}

\icmlcorrespondingauthor{Itay Lavie}{itay.lavie@mail.huji.ac.il}

\icmlkeywords{Gaussian process, NNGP, NTK, Transformer, Symmetry, Represenation theory, Deep learning theory, infinite width}

\vskip 0.3in
]



\printAffiliationsAndNotice{}  

\begin{abstract}
We study inductive bias in Transformers in the infinitely over-parameterized Gaussian process limit and argue transformers tend to be biased towards more permutation symmetric functions in sequence space. We show that the representation theory of the symmetric group can be used to give quantitative analytical predictions when the dataset is symmetric to permutations between tokens.
We present a simplified transformer block and solve the model at the limit, including accurate predictions for the learning curves and network outputs. We show that in common setups, one can derive tight bounds in the form of a scaling law for the learnability as a function of the context length. Finally, we argue WikiText dataset, does indeed possess a degree of permutation symmetry.
\end{abstract}

\section{Introduction}
Transformers show state-of-the-art performance on a wide variety of tasks~\cite{wolf_huggingfaces_2020,dosovitskiy_image_2021,chen_generative_2020,min_transformer_2022,brown_language_2020} with seemingly ever-improving performance~\cite{kaplan_scaling_2020,henighan_scaling_2020}. The past year has brought forth larger and more capable models than ever before~\cite{jiang_mixtral_2024,openai_gpt-4_2023,geminiteam_gemini_2023}, yet our understanding of them falls behind~\cite{goyal_inductive_2022,wen_transformers_2023}.

Recent works have advanced us in understanding specific aspects and behaviors like grokking~\cite{nanda_progress_2023,rubin_droplets_2023,liu_omnigrok_2022,liu_towards_2022}, in-context learning~\cite{von_oswald_transformers_2023,olsson_-context_2022}, and out-of-distribution (OOD) generalization~\cite{nam_achieving_2022,canatar_out--distribution_2021}. However, a unified view of the inductive bias of transformers is still lacking. It has been claimed that understanding and designing networks with better inductive bias is a necessary step toward AI~\cite{goyal_inductive_2022}; this can also make them safer and more suitable for deployment in high-risk situations (see for example \citet{Bommasani2021FoundationModels} and European Union AI act~\cite{europeancommission_proposal_2021}).

We approach the challenge from the infinitely over-parameterized Gaussian process (GP) limit, where the neural network (NN) becomes more analytically tractable but still shares many qualitative and quantitative similarities with finite NNs used in real life~\cite{lee_finite_2020,shankar_neural_2020,novak_bayesian_2018}. An established correspondence between infinitely wide NNs and Bayesian inference with a GP~\cite{neal_priors_1996,lee_finite_2020,welling_bayesian_2011,naveh_predicting_2021} allows us to identify the inductive bias of the NN in the GP limit with its Bayesian prior, thereby yielding a demystified concrete expression for the inductive bias. This approach has proven itself in the study of deep fully connected networks and convolutional neural networks (CNNs) by enabling accurate prediction of learning curves and explaining phenomena such as reduction of task complexity due to weight sharing~\cite{novak_bayesian_2018,naveh_predicting_2021} and CNN's utilization of hierarchical structure in the data~\cite{cagnetta_what_2023}.

In this work, we characterize the inductive bias by specifying how many samples will be required to learn a target function. We show that when the dataset possesses a permutation symmetry, learnability is closely tied to the irreducible representations (irreps) of the symmetric group. Namely, the more symmetric the function to permutations, as quantified below, the more learnable it is. 
We present an instructive example, work through it, and show we can accurately predict its outputs at the limit including performance under distributional shift. We then show that our predictions are a good approximation for wide but finite networks; these predictions again hold OOD. We present learnability bounds for richer networks, including a standard transformer block. Finally, we argue natural language (NL) does have some permutation symmetry, based on an analysis of WikiText-2~\cite{merity_pointer_2016}.

Our main contributions are: \vspace{-0.4cm}
\begin{itemize}
    \item We give explicit analytical predictions for the outputs and generalization performance of a NN with linear attention at the GP limit, in distribution and OOD. We show how irreducible representations of the symmetric group can be built and used for to predict learnability in this case.\vspace{-0.1cm}
    
    \item We extend our results to a transformer block with standard ${\rm softmax}$ attention. We show experimentally the learnability bounds found based on the dimension of the relevant irreducible representations are tight. \vspace{-0.1cm}
    
    \item We analyze WikiText-2 and show evidence for an approximate permutation symmetry in its principal components, suggesting that the toolbox presented can be of use in natural language datasets.
\end{itemize}
\vspace{-0.5cm}

\section{Methods}

This section presents the model we study and the tools used to analyze it at the GP limit. We show how to implement these tools to reveal the inductive bias in terms of the space of \emph{expressible} functions and their \emph{learnability} by the NN.\vspace{-0.3cm}

\subsection{Model}
\label{subsec:model}

\subsubsection{Neural network architecture}
\label{subsubsec:NN arch}
We study a transformer-like NN with one transformer block, for simplicity, we do not include residual connections or layer normalization, although these can be added. The NN is made of an embedding layer with added learned positional encoding (PE) $\vec{p}$, one multi-head self-attention layer (MHA), an MLP with one hidden layer and a final linear readout layer.

The input to the NN is made out of $L+1$ tokens $\vec{x}^a$ indexed by an upper sequence index $a=1,2,...,L+1$ with each token having an internal vocabulary dimension indexed by a lower index $i$. We group these with a Greek letter sample index $\mu=1,2,...,N$ into a rank $3$ tensor $X^a_{i,\mu}$, where we drop the sample index $\mu$ when we discuss only a single sample. One-hot encoding is used for the tokens, such that $ [\vec{x}^a]_i$ = $\delta_{i,v}$ where $v = 1,...,\nvoc$ is the token represented by $\vec{x}^a$.
Denoting the input by $x^a_i$ and the output of $l$'th layer by $z^{(l),a}_i$ the resulting NN is \vspace{-0.1cm}
\begin{equation}
\begin{aligned}
z_{i}^{(1),a} & =W_{ij}^{{\rm emb}}x_{j}^{a}+p_{i}^{a}\\
z_{i,h}^{(2),a} & =\Phi\left(\frac{Q_{j,h}^{a}K_{j,h}^{b}}{\sqrt{d_k}}\right)V_{i,h}^{b} \quad \text{(no $h$ summation)}\\
& = \Phi\left(\frac{W_{lm,h}^{Q}z_{m}^{(1),a}W_{ln,h}^{K}z_{n}^{(1),b}}{\sqrt{d_k}}\right)W_{ij,h}^{V}z_{j}^{(1),b}\\
z_{i}^{(3),a} & =W_{ij,h}^{O}z_{j,h}^{(2),a}\\
z_{i}^{(4),a} & =\phi\left(W_{ij}^{\left(4\right)}z_{j}^{(3),a}+b_{i}^{\left(4\right)}\right)\\
z_{i}^{(5),a} & =W_{ij}^{\left(5\right)}z_{j}^{(4),a}+b_{i}^{\left(5\right)}\\
\end{aligned}
\label{eq:NN_def}
\end{equation}
\begin{equation*}
\begin{aligned}
f_{i}^{a}(X)=z_{i}^{(6),a} & =W_{ij}^{{\rm d-emb}}z_{j}^{(5),a}
\end{aligned}
\end{equation*}
using Einstein's summation convention, with $\Phi$ and $\phi$ being some activation functions\footnote{A common choice would be $\Phi={\rm softmax}$ acting on the $b$ index and $\phi={\rm ReLU}$}. The NN parameters 
\begin{equation}
\begin{aligned}
W^{\rm emb} &\in \mathbb{R}^{d_{m} \times N_{\rm voc}} ,& \vec{p}^{\,a} &\in \mathbb{R}^{d_{m} } \\ W^Q,W^K,W^V &\in \mathbb{R}^{d_k \times d_{m}} ,& W^O &\in \mathbb{R}^{d_{m} \times d_k \times N_{\rm heads}} \\ W^{(4)} &\in \mathbb{R}^{ d_{ff} \times d_{m} } ,& \vec{b}^{(4)} &\in \mathbb{R}^{ d_{ff}} ,~ \vec{b}^{(5)} \in \mathbb{R}^{ d_{m}} \\ W^{(5)} &\in \mathbb{R}^{ d_{m} \times d_{ff} } ,&  W^{\rm d-emb} &\in \mathbb{R}^{ N_{\rm voc} \times d_{m} }
\end{aligned}
\end{equation}
are all learned. 
For the MHA we use $N_{\rm heads}$ heads and the same dimension $d_k=d_v=d_{m}/N_{\rm heads}$ for keys, queries, and values. Lastly, for the hidden layer $z^{(4)}$ we use dimension $d_{ff}$ which is of the same order of magnitude as the model dimension $d_{ff} \sim d_{m}$. 
Notably, consecutive affine transformations can be combined together without loss of generality, but they are kept in this way to align with standard notation\footnote{Combining such affine transformations would also induce a different prior in finite-sized NNs as shown in~\citet{li_statistical_2021}.}. 

As an instructive example, we will use a linearized MHA\footnote{similar to the one suggested by \citet{von_oswald_transformers_2023,hron_infinite_2020} and recently studied by \citet{ahn_linear_2024}.} $\Phi(x)=\frac{1}{L+1} x$ and linear MLP $\phi(x)=x$, as this setting allows for closed-form analytical predictions at the GP limit. Note that because we remove the common ${\rm softmax}$ non-linearity we add a division by the length to make sure the network's output stays $O(1)$ and does not scale with $L$.

\subsubsection{Task, Loss Function, Initialization, and Training Protocol}
\label{subsec:task,loss,init,traning}
The task is a pretraining task, namely, predicting the conditional probability for the next token given the context $p(\vec{x}^{L+2}|X) $. For simplicity, we limit the discussion to predicting the next token probability from a full context window of length $L+1$, and looking only at the prediction for the unknown token, meaning we define $f(X) := f^{L+1}(X)$.

Mean square error (MSE) loss with weight decay is used. The weights are initialized according to LeCun initialization, meaning the weights in each layer are i.i.d with $w\sim\mathcal{N}(0,\frac{1}{\sqrt{\rm{fan-in}}})$, and the biases are initialized to zero. For the convenience of the analytical calculations, we will initialize the PE as Gaussian i.i.d entries $p^a_i \sim \mathcal{N}(0,1/2)$ for $a\neq L+1$, for the last token we will initialize the PE to zero $p^{L+1}_i=0$. 

The training is done with uncorrected Langevin dynamics~\cite{neal_probabilistic_1993}, that is, gradient descent with noise $\eta \sim N(0,\sigma^2)$ added to the gradients, as a model for stochastic gradient descent. As shown in \citet{mingard_is_2020,liu_noise_2021,mandt_stochastic_2018} the result is indicative of the outcome of SGD training. We adjust the noise and weight decay such that following training with no data the network's weights distribution is the same as at initialization, as described in \citet{naveh_predicting_2021}. From a Bayesian perspective, such training protocol samples from the posterior distribution of a Bayesian NN with prior induced by the weights' initialization distribution. 

\subsubsection{Datasets}

We use a mixture of hidden Markov models (HMMs)~\cite{baum_statistical_1966} as a dataset. The mixture of HMMs is chosen for its balance between aspects of language, like long-range dependencies and sensitivity to (elementary) context~\cite{xie_explanation_2021}, and analytical tractability. This setting also yields a well-defined concept of distributional shift, as the NN can be trained on a fraction of the mixture and tested on another.

A HMM is composed of two stochastic processes, $h^a$ and $x^a$, where $a$ is the time-step index. The process $h^a$ is dubbed ``hidden" while $x^a$ is the observed process. The hidden process is Markovian, with $d_{\rm hidden}$ different states. The observed process depends only on the hidden state at the same time, where each of the possible $N_{\rm voc}$ outputs is given a different probability under each hidden state.

HMMs are conveniently described by stochastic emission and transition matrices. The $i,j$ entry of the transition matrix $T \in \mathbb{R}^{d_{\rm hidden} \times d_{\rm hidden}}$ represent the transition probability from the $j$'th hidden state to the $i$'th. Similarly, the $i,j$ entry of emission matrix $O \in \mathbb{R}^{N_{\rm voc} \times d_{\rm hidden}}$ represent the probability to emit the $i$'th output in the vocabulary when in the $j$'th hidden state. 

Our dataset is a mixture of HMMs with $N_{\rm voc}=2$ and $d_{\rm hidden}=2$, where the emission probabilities that define the HMM $p,q$ are themselves drawn from uniform distributions $p \sim U(p_a,p_a +w),~q \sim U(q_a,q_a +w)$. The transition probabilities are constant and deterministic. The transition and emission probabilities for a HMM in the mixture are given in matrix form by
\begin{equation}
T = \begin{bmatrix}0 & 1\\ 1 & 0 \end{bmatrix} ; \quad 
O = \begin{bmatrix} p & q\\ 1-p & 1-q \end{bmatrix}.
\label{eq:T-O-def}
\end{equation}
Finally, the initial hidden state, $h^1$, is a random variable with equal probability for each of the two possible hidden states. 

As a primer for the discussion to follow, we point out that the probability distribution defined by an HMM is invariant to permutation of tokens outputted under the same hidden state. 
We re-examine this point in section \ref{sec:results} and present evidence for an approximate permutation symmetry in the principal components of WikiText.

\subsection{Theory}
\label{sec:theory}
First, we present the correspondence between an infinite transformer-like NN and a Gaussian process (GP), known as the neural network Gaussian process correspondence (NNGP)~\cite{lee_deep_2018,hron_infinite_2020}.
To discuss the inductive bias in terms of functions rather than individual samples, we study the resulting GP averaged over datasets and its spectral decomposition. The null space is identified as the space of inexpressible functions, and the eigenvalues are interpreted as a measure of learnability for their corresponding eigenfunctions.
Capitalizing on the permutation symmetry of the kernel eigenvalue problem, we use tools from representation theory to provide upper bounds on the scaling of eigenvalues with context length.

As shown in \citet{hron_infinite_2020} when $d_k,N_{\rm heads}\to \infty$ the distribution of NN outputs induced by the initialization distribution converges to a GP $f(x)\sim\mathcal{GP}(0,k)$ with $k$ being the kernel of the NN. The kernel $k(X,Y)$ is the covariance between samples given by 
\begin{equation}
k(X,Y)=\E_\Theta \left[ f_\Theta(X) f_\Theta(Y) \right],
\label{eq:kernel_mat_def}
\end{equation}
where $\Theta$ stands for the NN parameters, drawn from the initialization distribution defined in \ref{subsec:task,loss,init,traning}. Following our training protocol, this correspondence carries over to the trained NN, such that the trained NN is equivalent to the result of Bayesian inference with the NNGP as prior~\cite{naveh_predicting_2021}.

While exact GP inference is generally hard, we can get an insight into the learning process by looking at the continuum limit where GP regression is averaged over all possible datasets of size $N$. 
For a large $N$, a continuum kernel can be used in an approximation known as the equivalent kernel (EK)~\cite{silverman_spline_1984,sollich_using_2004,cohen_learning_2021}. 
We can define the EK integral operator $\hat{K}$ associated with the continuum kernel,
\begin{equation}
\hat{K} u (X) = \E_{Y \sim p_{\rm train}} \left[ k(X,Y) u (Y) \right],
\end{equation}
where $p_{\rm train}$ is the training distribution. 

The inductive bias can now be understood by looking at the expression for the predictor (the output of the NN, averaged over the possible initializations) under the EK approximation on the eigenbasis of $\hat{K}$ 
\begin{equation}
\E_\Theta \left[ f_\Theta(X_*) \right] = \sum_{i=1}^\infty \frac{\lambda_i}{\lambda_i + \sigma^2/N} g_{i} \varphi_i (X_*)
\label{eq:EK-pred-diag}
\end{equation}
where $\sigma^2$ is the variance of the noise added to the gradients, or the observation uncertainty from a Bayesian perspective; $\varphi_i$'s are the eigenfunctions; $\lambda_i$'s are the corresponding eigenvalues and $g_i$ is the projection of $g(x)$ on $\varphi_i$ given by the inner product 
\begin{equation}
\langle g(x), \varphi_i (x) \rangle_x = \E_{x\sim \pdata} \left[ g(x) \varphi_i (x) \right].
\label{eq:inner-product-def}
\end{equation} 
We can give \eqref{eq:EK-pred-diag} an intuitive interpretation: The architecture and dataset dictate both the \emph{learnability} and \emph{expressibility}. All eigenfunctions corresponding to $\lambda = 0$ will not be expressible by the NN. For the expressible eigenfunctions, learnability for the eigenfunction $\varphi_i$ is dictated by the quantity
\begin{equation}
    L_i := \left| \frac{\langle f(x), \varphi_i(x) \rangle_{x \sim p}}{\langle y(x), \varphi_i(x) \rangle_{x\sim p}} \right| = \frac{\lambda_i}{\lambda_i+\sigma^2/N}.
    \label{eq:learnability_def}
\end{equation}
Where the last equality is given by using the EK predictor for $f(x)$ calculated on the same distribution $p(x)$ as the inner product, recovering the result of ~\cite{simon_eigenlearning_2023} (for effective ridge $\delta=\sigma^2$). With this matching EK predictor, the learnability $0\leq L \leq 1$ acts as a filter, passing almost nothing when $N \ll \sigma^2 \lambda^{-1}$ and passing the information from the target almost perfectly when $N \gg \sigma^2 \lambda^{-1}$. We see that learning the feature $\varphi_i(x)$ requires $N \simeq \sigma^2 \lambda^{-1}$ samples, predicting a performance improvement at that scale. 

Accordingly, predicting the learning of the NN is reduced to solving the eigenvalue problem for the EK kernel corresponding to the NN
\begin{equation}
\hat{K} \varphi_i (X) = \E_{Y \sim p_{\rm train}} \left[ k(X,Y) \varphi_i (Y) \right] = \lambda \varphi_i (X),
\label{eq:eigenproblem}
\end{equation}
and finding the projections of the target on the eigenbasis.

We now turn to use symmetry to simplify the eigenvalue problem. The fact that the NN defined in \ref{subsec:model}, never explicitly acts in sequence space (that is, the weights do not carry a sequence index) and the PE is drawn i.i.d guarantees a permutation symmetry between all the tokens but the last one\footnote{This is also true for the NTK~\cite{jacot_neural_2018} and we expect similar results to hold in that setting.}~\footnote{the last token is ``signaled out" at inference time as the only one who's output is desired. For a single transformer block, it is therefore sufficient to use only the last token as a query, and it is thereby not symmetric to the other tokens.}.

\subsubsection{Symmetry and representation theory}
\begin{figure}
    \centering
    \includegraphics[width=.68\columnwidth]{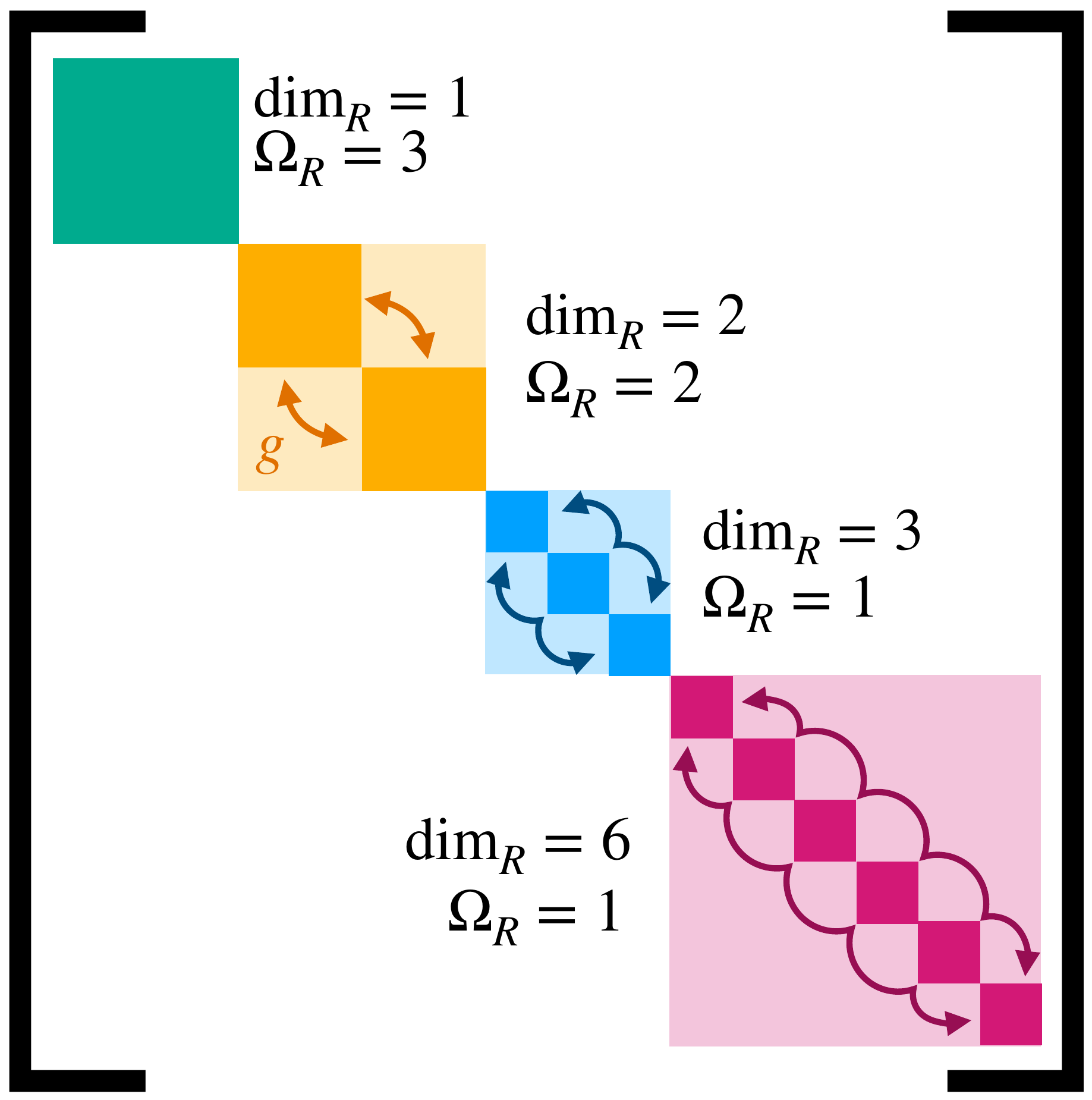}
     \caption{{\bf (Illustration of diagonalization using symmetries)} The figure illustrates the direct sum (block) structure described in Prop. \ref{prop:rep-theory}. Each color-shaded block represents an irrep, and each solid color represents a multiplicity block within the irrep. All elements outside the multiplicity blocks vanish, both between different irreps and within an irrep. 
     The symmetry actions $g \in G$ can mix multiplicity blocks as indicated by the arrows. Since all multiplicity blocks inside an irrep are linked by the symmetry actions they are all degenerate
     .}
     \label{fig:irrep_sym_illustartion}
\end{figure}
We start with an intuitive understanding of the role of symmetries and give a precise formulation later in this section. A fuller introduction and examples are given in Appendix~\ref{appendix:sym-group-irreps}. For a simple example where our use of representation theory amounts to a simple discrete Fourier transform, and introduction to permutation symmetry in appendix ~\ref{appendix:rep_fourier_intro}.

Symmetries can greatly simplify the eigenvalue problems like ~\eqref{eq:eigenproblem} above. We say an operator like $\hat{K}$ is symmetric under the action of a group $G$ if
\begin{equation}
    \forall g \in G, ~~ 
    k(\vec{x}_g, \vec{y}_g) = k(\vec{x}, \vec{y}) ~\&~  \pdata(\vec{x}_g)=\pdata(\vec{x}),
    \label{eq:intuitive_sym}
\end{equation}
where $\vec{x}_g$ is the result of acting with a symmetry action $g$ on $\vec{x}$, e.g. rotating $\vec{x}$ or permuting the entries of $\vec{x}$. 
Such an action is formalized through a \emph{representation} of the group, we give a precise definition in Prop.~\ref{prop:rep-theory}
.
A symmetry, as described in \eqref{eq:intuitive_sym}, means we are allowed to act with a symmetry action $g \in G$ but our model will stay invariant to this action. In the context of the eigenvalue problem in~\eqref{eq:eigenproblem}, such an action can be viewed as mixing different eigenfunctions $\varphi_i(x)$ (say by rotating the inputs $x$, such that the outputs $\varphi_i (\vec{x}_g)$ overlaps with $\phi_j(x)$ for $i\neq j$) without changing the eigenvalues.
This scenario implies, that all the eigenvalues of the mixed eigenfunctions must be identical, i.e. degenerate. Moreover, all eigenfunctions must either be members of such degenerate blocks, or in the simplest case be transformed back to themselves under the action of the symmetry group. See Fig~\ref{fig:irrep_sym_illustartion}.

If we study precisely how a symmetry group mixes the functions, we can identify the above-mentioned blocks in the space of expressible functions. The blocks would be a property of the symmetry group itself and would hold for any kernel satisfying \eqref{eq:intuitive_sym}.
Formally, the blocks correspond to the irreps of the group over the space of expressible functions (see Prop.~\ref{prop:rep-theory}). These can be understood as the minimal spaces of functions that mix with one another. The functions in those spaces cannot be "untangled" under the symmetry, hence the name irreducible.
%
%
\begin{proposition}
    Recalling results from \citet{tung_group_1985,fulton_representation_2004}.
    Given linear transformations $\{T_g | g \in G\}$ which constitute a representation of $G$ ($\forall g_1,g_2 \in G,~~ T_{g_1 g_2}=T_{g_1} T_{g_2}$) and a model symmetric under the action of a group $G$,
    i.e. 
    satisfying \eqref{eq:intuitive_sym} with $x_g = T_g x$. 
    {\bf It holds that}:
    The kernel operator can be decomposed into a direct sum, where each summand corresponds to an irrep of $G$ (shaded blocks in Fig.\ref{fig:irrep_sym_illustartion}). For an irrep $R$ that appears $\Omega_R$ times in $\hat{K}$ (said to have a multiplicity $\Omega_R$), each such block consists of $\Omega_R$ different eigenvalues, each with $m$-fold degeneracy, equal to the dimension of the irrep (${\rm dim}_R$).
    {\bf As a corollary}, each irrep of multiplicity $1$ gives exact eigenvectors of the kernel. For an irrep of multiplicity $\Omega_R$, finding the spaces of the irrep allows one to diagonalize in the $\Omega_R \times \Omega_R$ (multiplicity) space for each irrep individually; these spaces are guaranteed not to mix different irreps under the kernel.
\label{prop:rep-theory}
\end{proposition}

Going back to a more intuitive level, multiplicity means different sets of functions mix in the same way, but not between themselves. To separate these sets into eigenspaces the eigenvalue problem in the $\Omega_R \times \Omega_R$ multiplicity space needs to be solved in other means, but we are guaranteed we need to solve it in only one such multiplicity block, as all blocks are guaranteed to be degenerate (one solid color square of each color in Fig~\ref{fig:irrep_sym_illustartion}).

Degeneracy not only allows us to simplify the problem but also to give an asymptotic upper bound on the eigenvalues. Mercer's theorem~\cite{konig_eigenvalue_1986} guarantees $\hat{K}$ has a finite trace, which can be thought of as a fixed budget. Since all the eigenvalues are positive, they must share this fixed budget; leading to Prop.~\ref{prop:fixed_trace}.
\begin{proposition}
    Under the same conditions as Prop \ref{prop:rep-theory} and given the kernel is normalized, the trace is given by \vspace{-0.2cm}
    \begin{equation}
    \E_{x \sim \pdata} \left[ k(x,x) \right] \simeq 1 .
    \end{equation} 
    An eigenvalue $\lambda$ belonging to a space corresponding to an irrep $R$, is bound from above, $\lambda = O({\rm dim}_R^{-1})$ where ${\rm dim}_R$ is the dimension of $R$.
\label{prop:fixed_trace}
\end{proposition}

Since the kernel's trace is fixed to unity, and all eigenvalues in an irrep are degenerate (equal), the eigenvalues are upper bounded by one over the number of such eigenvalues, as they have to sum up to at most $1$.

We can now state formally the symmetry of our model as symmetry under the action of the symmetric group in $L$ symbols $S_L$ i.e. $k(T_{s_L}\,X, T_{s_L}\,Y)=k(X,Y)$ where $T_{s_L}$ is a representation of any element $s_L \in S_L$ that acts naturally on the sequence index~\footnote{We note this is a symmetry of the prior distribution and this is all that is required for our theory. The posterior distribution need not have this symmetry, as is often the case with learned positional encoding.}. Following Prop. \ref{prop:rep-theory},\ref{prop:fixed_trace} and the symmetry manifested in the model, we are interested in the irreps of the symmetric group.

Irreps of the symmetric group $S_L$ are uniquely labeled by partitions of $L$ to integers, written as ordered sets from the largest part to the smallest, such that the sum of the parts is $L$. For example, a partition to $L$ could be $(L-1,1)$ or $(L-4,2,1,1)$. To decompose the space of expressible functions we use the extensive literature on the representations of the symmetric group; a less formal introduction is given in Appendix~\ref{appendix:sym-group-irreps}, and a formal treatment is given in Appendix~\ref{appendix:irreps_theorem}.

Since the input is one-hot encoded, every target function will be a multilinear polynomial in the input tokens; that is, fixing all other variables we will remain with a linear function of $x_i^a$ for some particular $a,i$. This fact can be seen by considering each variable $x_i^a$ can only take on values $\{0,1\}$ so $(x_i^a)^n = x_i^a$ for $0< n \in \mathbb{Z}$. We thus wish to consider the decomposition of multilinear polynomials to irreps of the symmetric group.

\begin{restatable}{thm}{irrepsOfMultilinear}
\label{thm:irreps of multilinear}
The space of homogeneous multilinear polynomials in $n$ variables of degree $d$ can be \emph{fully} decomposed into $\min \{d+1,n-d+1 \}$ \emph{unique} irreps of $S_n$ labeled by the partitions $(n-k,k)$ for $0 \leq k \leq d,n-d$. 
\end{restatable}

See proof in appendix \ref{appendix:irreps_theorem}. We can therefore expand any analytic function into polynomials and decompose them into the irreps of the symmetric group.

The dimension of the $k$'th irrep (${\rm dim}_k$) of the form $(L-k,k)$ scale as ${\rm dim}_k = \frac{L!}{k! \frac{(L-k+1)!}{L-2k+1}} \sim L^k$. We can now  quantitatively define a measure for symmetry to permutations: the more symmetric a function is, the less it may mix with other functions, and the smaller the dimension of the irreps it belongs to (smaller $k$). 
Going back to the definition of learnability in~\eqref{eq:learnability_def} we see that the number of samples required to learn a function in the representation $(L-k,k)$ is asymptotically bounded from below by\vspace{-0.2cm}
\begin{equation}
    N \simeq \lambda^{-1}_{(L-k,k)} \sigma^2 = \Omega(L^{k}).
\end{equation}
We therefore see that the more symmetric a function is to permutations (smaller $k$) the more learnable it is.

\subsection{Example: Linear Activations}
In this example, we choose $\Phi(x)=\frac{1}{L+1} x$ and linear MLP $\phi(x)=x$, as previously noted in \ref{subsubsec:NN arch} and solve the eigenvalue problem presented in the previous section. Note the linear activation functions $\Phi,\phi$ do not imply a linear NN as the attention layer is inherently non-linear. While this example is a minimal transformer like NN, our dataset already goes beyond the landscape of complete permutation invariance and demonstrates how the tools presented above can be adapted to richer datasets where the permutation invariance is partially broken.

\subsubsection{Expressibility}
\label{subsec:Expressibility}
First, we want to identify the space of functions spanned by $\varphi_i$ with $\lambda_i\neq 0$, the space of expressible functions. 
\begin{claim}
    The space of functions expressible by the model stated in section \ref{subsec:model} is spanned by the linear functions of $\{x^a_1\}_{a=1}^L$ multiplied by linear functions of $x^{L+1}_1$, which is a $2L+2$ dimensional space.
\end{claim}
The kernel function corresponding to our NN is given by
\vspace{-0.3cm}
\begin{equation}
k(X,Y)= \frac{1}{8} \vec{x}^{L+1} \cdot \vec{y}^{L+1} \frac{1}{(L+1)^2} \sum_{a,b=1}^{L+1} \left( \vec{x}^{a} \cdot \vec{y}^{\,b} + \delta^{a,b} \right)^2 .
\label{eq:full kernel}
\end{equation}
One-hot encoding not only implies multilinearity of the outputs, but also guarantees multilinearity of the kernel in the inner product of two vectors $(\vec{x}^a \cdot \vec{y}^{\,b})^n = (\vec{x}^a \cdot \vec{y}^{\,b})$ for $0< n \in \mathbb{Z}$. In this example, it means only linear terms in the context window $a,b=1,...,L$ are present.

We can further restrict the model's expressibility in our case, by considering large context windows $L \gg 1$. In that case, we can approximate the kernel given in \eqref{eq:full kernel} by
summing only up to $L$, and dropping sub-leading contributions in $\frac{1}{L}$. We show these indeed give only sub-leading corrections in appendix \ref{appendix:corrections-xL+1}. 
These simplifications result in
\begin{equation}
    k(X,Y)=  \underbrace{\vec{x}^{L+1} \cdot \vec{y}^{L+1}}_\mathfrak{A} \frac{1}{8 L^2} 
    \underbrace{\sum_{a,b=1}^{L} \left( \vec{x}^{a} \cdot \vec{y}^{\,b} + \frac{\vec{x}^{a} \cdot \vec{y}^{\,a}}{L} +\frac{1}{L} \right)}_\mathfrak{B} .
\label{eq:simplified-kernel}
\end{equation}
Finally, Since our particular model uses a vocabulary of size $2$ the entries of a one-hot vector are completely determined by one another $x^a_2 = 1-x^a_1$, allowing us to write it using only the first entry.
As can be seen in \eqref{eq:eigenproblem}, the only $X$ dependence in the l.h.s comes from the kernel $k(X,Y)$, thus for the equality to hold for every $X$, the eigenfunction $\varphi_i (X)$ of $\lambda_i \neq 0$ must be in the space of functions spanned by $k(X,\cdot)$. For example, if $k(X,Y)$ is linear in $X$ only linear functions will be expressible.
Based on this argument (formally given by the representer theorem~\cite{scholkopf_generalized_2001}), we may conclude the space of expressible functions is spanned by linear functions of $\{x^a_1\}_{a=1}^L$ multiplied by linear functions of $x^{L+1}_1$; a space of dimension $2L+2$.

\subsubsection{Learnability}
\label{subsec:Learnability}
Moving from expressibility to learnability requires knowledge of the full spectrum of the kernel. While this problem is generally hard, we will use the representation theory tools developed above to simplify it. We would like to decompose the zeroth and first-degree multilinear polynomial to irreps.

\begin{claim}
    The space of expressible functions of the model described above can be decomposed into irreps as follows.
    A trivial representation (${\rm dim} = 1$) of multiplicity $6$ and a standard representation 
    (${\rm dim} = L/2-1$) of multiplicity $4$.
\end{claim}

Starting from the largest structure, notice the kernel is a product of two terms ($\mathfrak{A},\mathfrak{B}$ in \eqref{eq:simplified-kernel}).
The $\mathfrak{A}$ part is diagonalized
in the basis 
\begin{equation}
a(\vec{x}^{L+1}) = x^{L+1}_1,~~b(\vec{x}^{L+1}) = (1-x^{L+1}_1),
\end{equation}
which leads to multiplicity $2$ for all irreps involved; one belonging to the $a$ block and one to the $b$ block of $\mathfrak{A}$.

Moving on to the $\mathfrak{B}$ term, as expected from the general argument presented in the previous section, we find it is symmetric under the action of the permutation in the symmetric group $S_L$ on the set of tokens in the context window $\{ x^a \}_{s=1}^L$. The full $S_L$ symmetry is not, however, presented in the probability distribution of our chosen dataset, as tokens have different emission probabilities under different hidden states. Nevertheless, a smaller symmetry is preserved, allowing permutations only within the same hidden states. Since the transition between hidden states is deterministic, we find that all odd (even) tokens belong to the same hidden state and can be permuted between themselves, giving rise to the smaller symmetry group $S_{L/2}^{\rm odd} \times S_{L/2}^{\rm even} := \mathcal{S}$ \footnote{Assuming $L$ is even for simplicity} as a symmetry of $\hat{K}$.

As discussed in the previous subsection, in this example, only polynomials up to first degree can have non-vanishing eigenvalues. First degree polynomials are decomposed to two irreps (see theorem \ref{thm:irreps of multilinear}), namely the trivial $(L/2)$ and standard representation $(L/2-1,1)$. The trivial representation has dimension $1$ with multiplicity $4$: one for the even subspace and one for the odd subspace, times $2$ for the multiplicity from $\mathfrak{A}$. The standard representation has dimension $L/2-1$ with multiplicity $4$, broken down in the same way as before. For zeroth degree polynomials (constants) only the trivial representation exists, of multiplicity $2$: again, the coming from $\mathfrak{A}$. The process of concretely writing down the functions that span the space can be simplified further using the cyclic 
sub-group $C_{L/2}$, this process is details in appendix \ref{appendix:linear_example}.

Using symmetries and the partition to $\mathfrak{A},\mathfrak{B}$ we are able to reduce the eigenvalue problem to two\footnote{One for the $a$ block and one form the $b$ block} $3 \times 3$ spaces of the trivial representation, which are diagonalizable in closed form, and a diagonalized $2L-4$ dimensional space of the standard representation.
We can repeat the same procedure for polynomials of any order and decompose them to irreps (see appendix \ref{appendix:sym-group-irreps} for a discussion of the method, and an example); thereby allowing us to expand the results to a wider class of NNs including non-linear and deeper NNs.

Gathering the results of this section, \textbf{Given}:
     (1) \eqref{eq:EK-pred-diag}, together with the (2) learnable target given in \eqref{eq:learning-target-coeff}, the (3) eigendecomposition given in equations \ref{eq:k_vectors}, \ref{eq:appndx_k_eigvals}, and the (4) eigendecomposition of the two $3 \times 3$ spaces spanned by the basis in \eqref{eq:phi_base_def}.      
\textbf{One can}
    predict accurately the output of the model described in section \ref{subsec:model} with linear activation functions in the GP limit. Additionally, One can make accurate predictions for the generalization loss, even under a distributional shift.

\begin{figure*}[t]
    \centering
\includegraphics[width=.68\columnwidth]{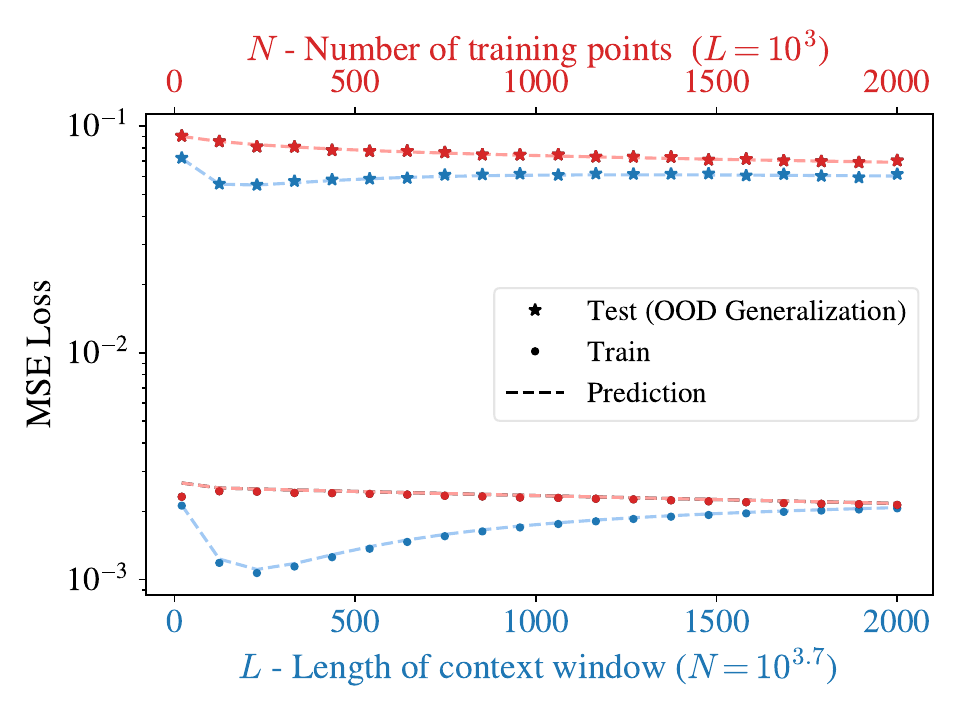}
\vspace{-0.35cm}\hfill
\includegraphics[trim = 60mm 5mm 60mm 20mm, clip,width=.68\columnwidth]{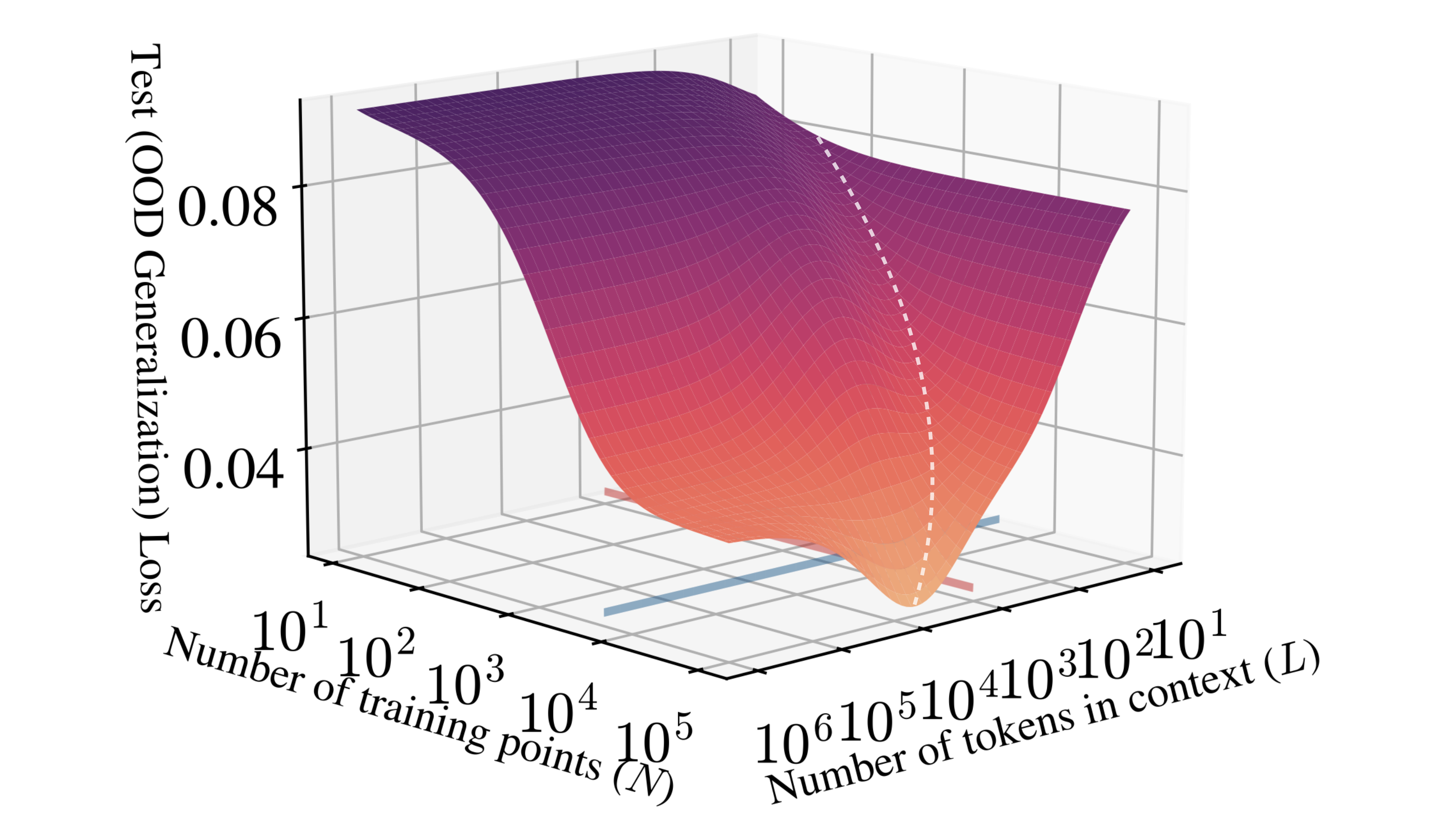}
\vspace{-0.35cm}\hfill
\includegraphics[width=.68\columnwidth]{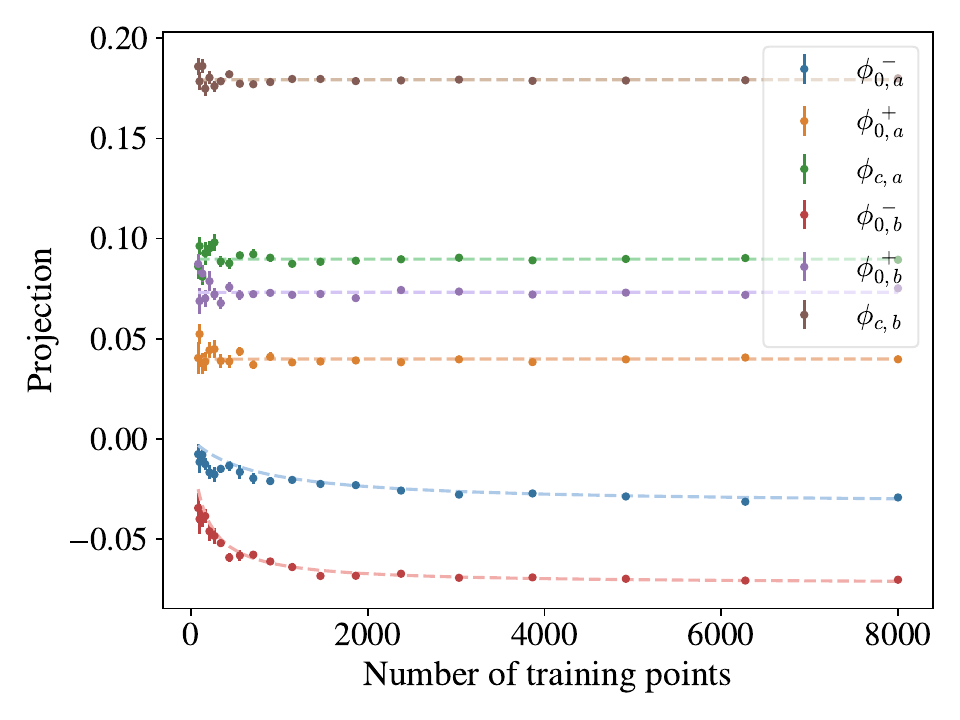}
\vspace{0.7cm} 
     \caption{{\bf Left: (theory vs. experiment)} Two sections, along constant $N$ (in blue) and $L$ (in red), of the MSE loss as shown in the center. We find good agreement between our theoretical predictions (calculated for the train and test distributions) and exact inference with a GP, equivalent to inference with an infinitely wide NN. 
     Stars indicate the experimental MSE loss calculated on the test dataset, where the majority of samples are OOD w.r.t to training dataset.
     {\bf Center: (learnability scaling law)} Prediction for the generalization MSE. We see the learnability threshold as a diagonal valley (marked by the dashed line) of constant $N/L$ ratio as a consequence of having target eigenvalues of scaling $\lambda \sim L^{-1}$. This is the onset of the regime where there are enough datapoints to use the full potential of the context length. 
     {\bf Right: (linear MLP predictions extend to non-linear MLP)} Projections of the exact GP inference predictor with \emph{non-linear MLP} on the basis vectors $\{ \varphi \}$ specified in \eqref{eq:phi_base_def}, together with the theoretical prediction for the projection for a GP with \emph{linear MLP} as a function of the number of training points. The predictions for the transformer block with linear attention and linear MLP are useful for non-linear MLP as well.}
    \label{fig:linear}
\end{figure*}
\section{Experimental Results}
\label{sec:results}

In this section, the theory is compared to numerical experiments. We start by comparing our predictions for the example of linear activation functions with exact Bayesian inference using the NNGP prior, which is equivalent to the output of a NN train with Langevin dynamics. We predict the performance OOD and show good agreement with experiments. The $N$-$L$ scaling law is shown and the kernel for ${\rm erf}$ activation function is shown to give a similar predictor per $N$. We then move on to real, wide but finite, NN with non-linear activation functions trained by SGD. We show nontrivial predictions such as OOD performance carryover from the linear MLP to the non-linear MLP case. We then present the kernel's spectrum of an NN with standard ${\rm softmax}$ attention and show that the scaling law bounds derived on the eigenvalues become tight. Lastly, we analyze WikiText-2 and show that at leading order correlations the dataset does indeed appear to be permutation symmetric to a good approximation.

\subsection{Linear Activation Functions}
On the left of Fig. \ref{fig:linear} the predictions for the loss as a function of $N$ and $L$ are presented, together with exact Bayesian inference, showing good agreement both on train ($p \sim U(0.4,0.4+10^{-1.5}),~q \sim U(0.5,0.5+10^{-1.5})$) and test ($p,q \sim U(0,1)$) distribution loss. For details on the analytical predictions for the loss see Appendix~\ref{appendix:OOD-MSE}.

The $2d$ manifold of the analytical prediction for the (OOD generalization) loss computed on a mixture of HMMs drawn from $p,q \sim U(0,1)$, as a function of $N,L$ is presented in Fig. \ref{fig:linear} center.
As shown in Appendix \ref{appendix:linear_example}, only functions that belong to the trivial representation contribute to spanning the target. Four out of six corresponding eigenvalues scale as $\lambda \sim 1$ and the other two scales as $\lambda \sim L^{-1}$. Resulting in the scaling law $N \simeq 1/ \lambda \propto L$ 
This result appears as a connected valley of minimal loss when holding $N$ constant, i.e. a constant $N \propto L$ ratio. 

Lastly, we compare the GP predictor for a network with non-linear MLP $\phi(x)={\rm erf} (x)$ (in reference to \eqref{eq:NN_def}) to that of the linear example solved above.
To the right in Fig \ref{fig:linear}, the exact GP inference (computed for  \emph{non-linear} MLP) predictions projected on to the $\{ \varphi \}$ basis vectors in \eqref{eq:phi_base_def} are compared with analytical predictions for a network with \emph{linear} MLP as a function of the number of samples in the dataset. They match very well 
in terms of learnability as seen by the rate at which the projection coefficients reach their asymptotic value
. 
We see the NN with non-linear MLP follows the same trend and learns a similar predictor to the one with linear MLP.

\vspace{-0.2cm}
\begin{figure*}[t]
    \centering
\includegraphics[width=.68\columnwidth]{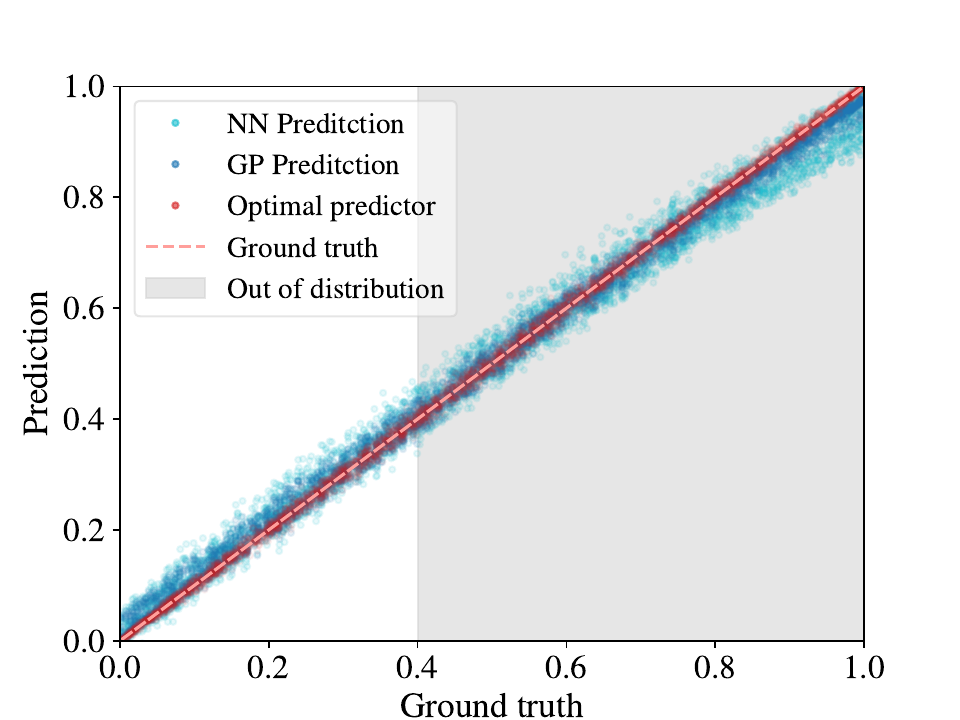}
  \vspace{-0.35cm}\hfill
\includegraphics[width=.68\columnwidth]{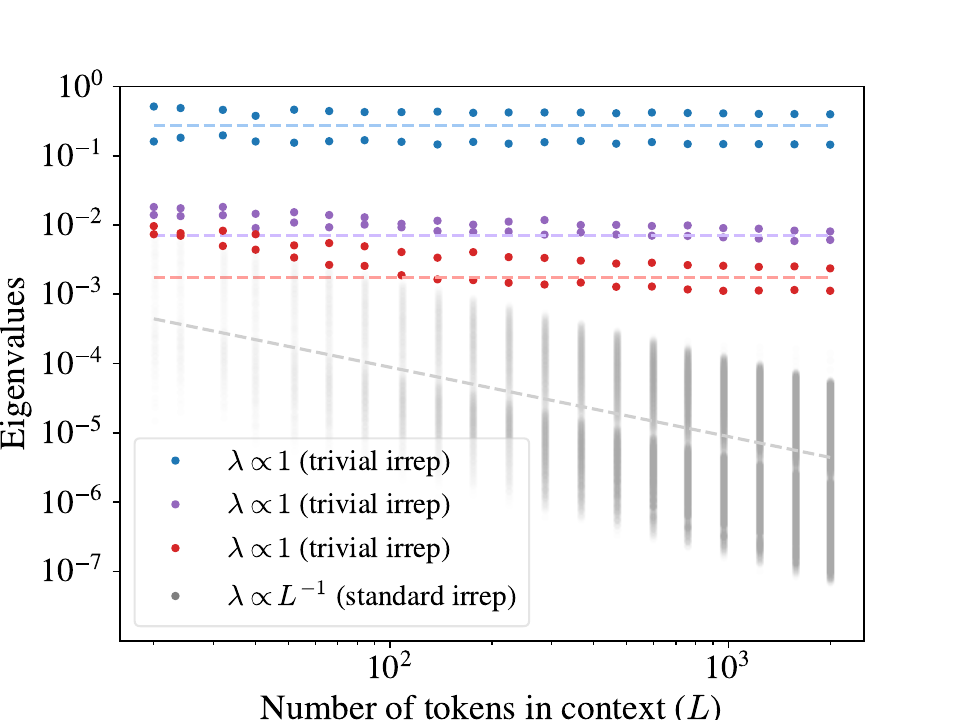}
  \vspace{-0.35cm}\hfill
\includegraphics[width=.68\columnwidth]{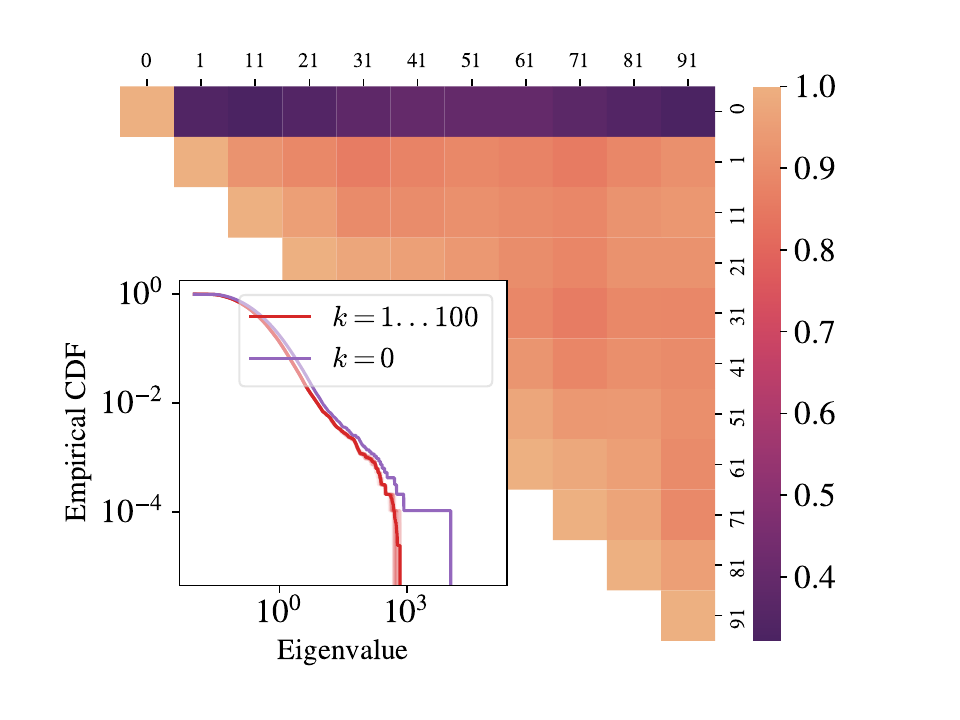}
\vspace{0.7cm} 
     \caption{ {\bf Left: (Predicted OOD generalization matches experiment)} Finite NN and exact GP inference (cyan and blue respectively), optimal predictions (red) and ground truth (dashed light red) as a function of the ground truth target for the sample. We see good agreement between predictions and experiment in distribution and OOD. The performance OOD is only slightly worse than in distribution, as indicated by the spread around the ground truth.
     {\bf Center: (the scaling bounds are tight for eigenvalues of softmax attention)} 
    The spectrum of the empirical kernel of a network with ${\rm softmax}$ attention as a function of the context length ($L$). The scaling with $L$ is bound tightly by the scaling deduced from the dimension of the corresponding irrep of the symmetric group. The light dashed lines serve only as a guide to the eye for the scaling law; they are not predictions for specific values.
    {\bf Right: (evidence for approximate permutation symmetry in WikiText)} The triangle shows the cosine similarity induced by the Frobenius inner product between the linear features of WikiText $C^{kk}$ and $C^{k'k'}$ for the $k$'s indicated on the boundary. We see all sampled $k \neq 0$ are similar to one another but different from $k=0$ as predicted by the irrep decomposition. The Empirical CDF plot shows the CDF for the eigenvalues of those sampled matrices. Different $k$'s for $k\neq0$ are almost identical. $k=0$ has a distinct distribution.}
    \label{fig:erfy}
\end{figure*}
\subsection{Finite Neural Networks and Natural Language}
While the example of linear activation functions can be fully understood one may wonder whether the conclusions carry over to more common scenarios, here we address these concerns. The doubts we will try to dispel are \vspace{-0.2cm}
\begin{itemize}
\item While there is an established correspondence between Bayesian inference with a GP and an infinite NN, does a finite NN, trained with SGD with a finite learning rate show similar results?
\item Could it be that the OOD generalization is a result of the limited expressibility and or the GP limit? 
\item Are the bounds derived using the dimension of the corresponding irreps tight?
\item Among other tasks, transformers are widely used in NLP, is NL permutation symmetric to any degree? 
\end{itemize}
The first result involves again a non-linear MLP with ${\rm erf}$ activation function and linearized MHA. The NN is trained on $8,000$ samples drawn from the mixture $p,q \sim U(0,0.4)$ with SGD with a mini-batch size of 50 and a learning rate of $10^{-3}$ for $10,000$ epochs.
In figure \ref{fig:erfy} left, the output of the NN, together with exact Bayesian inference with the NNGP prior, the optimal predictor, and the ground truth, are shown as a function of the ground truth target value; each point on the plot can be thought of as correlation between the output and the true value. The performance OOD only starts to decline slightly around $0.8$, matching well with the exact GP predictions despite the finite size and SGD training. It remains to be seen whether these results hold for a large transformer.

In the center panel of Fig.~\ref{fig:erfy} we see the spectrum of the kernel matrix as defined in \eqref{eq:kernel_mat_def}, for a NN with $\Phi = {\rm softmax}$ and $\phi(x)=x$. The scaling of the eigenvalues is improved such that the eigenvalues in the trivial irrep scale as $L^0$ and the eigenvalues in the standard irrep scale as $L^{-1}$, meaning, they take the maximum scaling possible based on the degeneracy of the irrep. Comparing these values to those of a fully connected network (FCN) can give a view from infinity on the advantages transformers have over FCNs. For our dataset and task, the target function belongs to the trivial representation of the symmetric groups, this means that the network with $\Phi = {\rm softmax}$ can learn the target in $N=\Theta(1)$ samples. On the other hand, an FCN will need $N=O(L)$ samples, as its symmetry group would be the sign-symmetric group, where all $L$ linear functions are degenerate~\cite{yang_fine-grained_2020}. Such a result extends further to settings used to study in-context learning~\cite{von_oswald_transformers_2023}, as they too aggregate symmetrically over context, allowing transformers to gain significant advantage over FCNs. 

Finally, we present some evidence suggesting NL does possess an approximate permutation symmetry, at least up to linear correlations.
We examine the (first order) correlations in WikiText-2
at the basis of the cyclic permutation irreps (for experimental details see appendix \ref{appendix:wikitext}) \vspace{-0.2cm}
\begin{align}
C^{k k'}_{ij} &:= \E_{X \sim {\rm WikiText-2}} \left[ X^a_i V^{a k} X^b_j V^{b k'} \right]; \\
V^{ak} &:= \exp(i \frac{2 \pi}{L} a k),~~  \vspace{-0.3cm}
\begin{aligned} \vspace{-0.2cm}
    a &= 1,...,L \\
    k &= 0,...,L-1
\end{aligned} .\vspace{-0.4cm}
\end{align}
If permutation symmetry were to hold, we would expect all $C^{k k}$ correlation matrices with $k \neq 0$ to be interchangeable, as they are all part of the standard irrep. We quantify this quality by the cosine similarity induced by Frobenius inner product and by their spectrum. As shown in Fig. \ref{fig:erfy} right, there is indeed a large similarity in the standard irrep. This similarity does not exist with the trivial irrep ($k=0$). The spectrum of the different correlation matrices inside the standard irrep is almost identical as well, as indicated by the eigenvalue CDF in the same figure. This similarity, again, does not exist between the two irreps (i.e. $k=0,k\neq 0$).

\section{Discussion}
\label{sec:discussion}

In this work, we analyzed a family of transformer-like models and showed that their inductive bias can be understood using the representation theory of the symmetric group when the dataset possesses permutation symmetry. In this setting, we derived a scaling law for the number of data samples required to learn a target as a function of the context length. We showed these predictions are generic and hold with added non-linear activation functions, as well as OOD.

Critically, the above results depend on a permutation symmetric dataset\footnote{Though the solved linear example shows permutation symmetry between subsets in the sequence is also powerful} while some settings do have this exact symmetry like transformers for sets~\cite{lee_set_2019,kim_transformers_2021} or setting studied for ICL~\cite{power_grokking_2022,von_oswald_transformers_2023}, natural language does not seem to have it prima facie. We have shown that, in fact, first-order correlations in WikiText-2 seem to largely manifest this symmetry. This means that when learning linear targets or up to $O(L)$ samples, such models will be bound by the scaling laws discussed above. One such linear function (in the context tokens) that is relevant to NLP is the copying heads discussed in \citet{olsson_-context_2022}, while induction heads would be second order in the context tokens. This fact motivates examining the correlations in NL to second order, as a concrete mechanism for in-context learning can already appear there; we leave this for future work.

Notwithstanding, \emph{some} conclusions can be carried from a symmetric dataset to an arbitrary one. Assume the target is one of the eigenfunctions of the kernel when diagonalized on a permutation symmetric measure $p(x)$, namely $y(x)=\varphi_t(x)$, we may substitute into~\eqref{eq:learnability_def} a predictor $f(x)$ attained by inference with the same GP on a different, possibly non-symmetric dataset $D_N=\{x_\mu,y(x_\mu)\}_{\mu=1}^N$. The learnability with the new $f(x)$ can nevertheless be bounded \cite{future_work}
\begin{equation}
    L_t\leq \frac{\lambda_t N \E_{x\sim D_N} [\phi_t^2 (x)]} {\sigma^2}.
\end{equation}
suggesting $N\approx \lambda_t^{-1} \sigma^2/\E_{x\sim D_N} [\phi_t^2 (x)] $ samples will be required to learn $\phi_t$ from $D_N$ when measuring learnability on $p(x)$.

Lastly, while our work accounts for the implicit inductive bias of the architecture, it does not address other sources of inductive bias, like finite learning rate~\cite{lewkowycz_large_2020,beugnot_benefits_2022,mohtashami_special_2023} and finite size corrections to the GP limit. As recent works have shown~\cite{seroussi_separation_2023,bordelon_self-consistent_2023,segadlo_unified_2022,li_statistical_2021,pacelli_statistical_2023,dyer_asymptotics_2019,aitchison_deep_2021}, the GP limit is used as a starting point for more advanced methods that study finite size corrections and capture important phenomena like representation learning. Studying such corrections is left to future work.

\subsection{Related Works}


\emph{Inductive bias}. The term spectral bias was coined by~\citet{rahaman_spectral_2019}, referring specifically to the Fourier spectrum. The term has since been used to describe other forms of inductive bias, including those coming from the spectrum of the corresponding kernel (eigenvalue decomposition)~\cite{bordelon_spectrum_2020,cao_towards_2021,canatar_spectral_2021}. 
The form of inductive bias studied in this paper has been extensively studied for fully connected networks (FCNs): for data uniformly distributed on a hypersphere~\cite{basri_convergence_2019,bietti_deep_2021,azevedo_sharp_2014,scetbon_spectral_2021} and non-uniformly~\cite{basri_frequency_2020}, and for Gaussian and uniform distribution on the hypercube~\cite{yang_fine-grained_2020}. \citet{bietti_sample_2021} have studied how additional symmetries above the rotational symmetry of FCNs can reduce sample complexity and \citet{tahmasebi_exact_2023} generalized these results.
In the context of CNNs symmetries were used to design better inductive bias (see~\citet{cohen_group_2016}), and their corresponding kernels' spectra were analyzed as well~\cite{bietti_approximation_2021,xiao_eigenspace_2022,cagnetta_what_2023,geifman_spectral_2022} showing how their inductive bias is beneficial for learning a range of tasks.
To the best of our knowledge, our work is the first to study inductive bias of transformers from the perspective of symmetries in the infinite width (GP) limit. Recently an empirical work has characterized some aspects of the inductive bias of transformers~\cite{bhattamishra_simplicity_2023}; our work complements it and can be seen as a theoretical explanation. Lastly, ~\citet{fu_what_2023} has studied random feature attention under limiting assumptions (no positional encoding, limited depth); our work approaches the challenge from a the perspective of symmetry \& representation theory and its results are valid for a more general family of models. 

\emph{Understanding transformers}. A large body of knowledge has accumulated regarding understanding specific phenomena related to transformers like grokking~\cite{liu_towards_2022,liu_omnigrok_2022,rubin_droplets_2023,nanda_progress_2023}, in-context learning~\cite{olsson_-context_2022,garg_what_2022,akyurek_what_2022,von_oswald_transformers_2023}. Other works focused on other aspects like signal propagation~\cite{cowsik_geometric_2024,noci_signal_2022} and OOD generalization~\cite{nam_achieving_2022}. Our work focuses directly on the transformers' inductive bias and characterizes it at the GP limit.

\clearpage
\section*{Acknowledgements}

Z.R. and I.L. acknowledge support from ISF Grant 2250/19.

\section*{Impact Statement}
This paper presents work whose goal is to advance the field of Machine Learning. There are many potential societal consequences of our work, none which we feel must be specifically highlighted here.


\bibliography{manual,references1}

\begin{thebibliography}{87}
\providecommand{\natexlab}[1]{#1}
\providecommand{\url}[1]{\texttt{#1}}
\expandafter\ifx\csname urlstyle\endcsname\relax
  \providecommand{\doi}[1]{doi: #1}\else
  \providecommand{\doi}{doi: \begingroup \urlstyle{rm}\Url}\fi

\bibitem[Ahn et~al.(2024)Ahn, Cheng, Song, Yun, Jadbabaie, and
  Sra]{ahn_linear_2024}
Ahn, K., Cheng, X., Song, M., Yun, C., Jadbabaie, A., and Sra, S.
\newblock Linear attention is (maybe) all you need (to understand transformer
  optimization), March 2024.
\newblock URL \url{http://arxiv.org/abs/2310.01082}.
\newblock arXiv:2310.01082 [cs, math].

\bibitem[Aitchison et~al.(2021)Aitchison, Yang, and Ober]{aitchison_deep_2021}
Aitchison, L., Yang, A., and Ober, S.~W.
\newblock Deep kernel processes.
\newblock In \emph{Proceedings of the 38th {International} {Conference} on
  {Machine} {Learning}}, pp.\  130--140. PMLR, July 2021.
\newblock URL \url{https://proceedings.mlr.press/v139/aitchison21a.html}.
\newblock ISSN: 2640-3498.

\bibitem[Akemann et~al.(2015)Akemann, Baik, and
  Di~Francesco]{10.1093/oxfordhb/9780198744191.001.0001}
Akemann, G., Baik, J., and Di~Francesco, P.
\newblock \emph{The {Oxford} {Handbook} of {Random} {Matrix} {Theory}}.
\newblock Oxford University Press, September 2015.
\newblock ISBN 978-0-19-874419-1.
\newblock \doi{10.1093/oxfordhb/9780198744191.001.0001}.
\newblock URL \url{https://doi.org/10.1093/oxfordhb/9780198744191.001.0001}.

\bibitem[{Akyürek} et~al.(2022){Akyürek}, Schuurmans, Andreas, Ma, and
  Zhou]{akyurek_what_2022}
{Akyürek}, E., Schuurmans, D., Andreas, J., Ma, T., and Zhou, D.
\newblock {What} learning algorithm is in-context learning? {Investigations}
  with linear models.
\newblock September 2022.
\newblock URL \url{https://openreview.net/forum?id=0g0X4H8yN4I}.

\bibitem[Azevedo \& Menegatto(2014)Azevedo and Menegatto]{azevedo_sharp_2014}
Azevedo, D. and Menegatto, V.~A.
\newblock Sharp estimates for eigenvalues of integral operators generated by
  dot product kernels on the sphere.
\newblock \emph{Journal of Approximation Theory}, 177:\penalty0 57--68, January
  2014.
\newblock ISSN 0021-9045.
\newblock \doi{10.1016/j.jat.2013.10.002}.
\newblock URL
  \url{https://www.sciencedirect.com/science/article/pii/S0021904513001652}.

\bibitem[Basri et~al.(2019)Basri, Jacobs, Kasten, and
  Kritchman]{basri_convergence_2019}
Basri, R., Jacobs, D., Kasten, Y., and Kritchman, S.
\newblock The {Convergence} {Rate} of {Neural} {Networks} for {Learned}
  {Functions} of {Different} {Frequencies}.
\newblock In \emph{Advances in {Neural} {Information} {Processing} {Systems}},
  volume~32. Curran Associates, Inc., 2019.
\newblock URL
  \url{https://proceedings.neurips.cc/paper_files/paper/2019/hash/5ac8bb8a7d745102a978c5f8ccdb61b8-Abstract.html}.

\bibitem[Basri et~al.(2020)Basri, Galun, Geifman, Jacobs, Kasten, and
  Kritchman]{basri_frequency_2020}
Basri, R., Galun, M., Geifman, A., Jacobs, D., Kasten, Y., and Kritchman, S.
\newblock Frequency {Bias} in {Neural} {Networks} for {Input} of
  {Non}-{Uniform} {Density}.
\newblock In \emph{Proceedings of the 37th {International} {Conference} on
  {Machine} {Learning}}, pp.\  685--694. PMLR, November 2020.
\newblock URL \url{https://proceedings.mlr.press/v119/basri20a.html}.
\newblock ISSN: 2640-3498.

\bibitem[Baum \& Petrie(1966)Baum and Petrie]{baum_statistical_1966}
Baum, L.~E. and Petrie, T.
\newblock Statistical {Inference} for {Probabilistic} {Functions} of {Finite}
  {State} {Markov} {Chains}.
\newblock \emph{The Annals of Mathematical Statistics}, 37\penalty0
  (6):\penalty0 1554--1563, 1966.
\newblock ISSN 0003-4851.
\newblock URL \url{https://www.jstor.org/stable/2238772}.
\newblock Publisher: Institute of Mathematical Statistics.

\bibitem[Beugnot et~al.(2022)Beugnot, Mairal, and Rudi]{beugnot_benefits_2022}
Beugnot, G., Mairal, J., and Rudi, A.
\newblock On the {Benefits} of {Large} {Learning} {Rates} for {Kernel}
  {Methods}.
\newblock In \emph{Proceedings of {Thirty} {Fifth} {Conference} on {Learning}
  {Theory}}, pp.\  254--282. PMLR, June 2022.
\newblock URL \url{https://proceedings.mlr.press/v178/beugnot22a.html}.
\newblock ISSN: 2640-3498.

\bibitem[Bhattamishra et~al.(2023)Bhattamishra, Patel, Kanade, and
  Blunsom]{bhattamishra_simplicity_2023}
Bhattamishra, S., Patel, A., Kanade, V., and Blunsom, P.
\newblock Simplicity {Bias} in {Transformers} and their {Ability} to {Learn}
  {Sparse} {Boolean} {Functions}.
\newblock In \emph{Proceedings of the 61st {Annual} {Meeting} of the
  {Association} for {Computational} {Linguistics} ({Volume} 1: {Long}
  {Papers})}, pp.\  5767--5791, Toronto, Canada, 2023. Association for
  Computational Linguistics.
\newblock \doi{10.18653/v1/2023.acl-long.317}.
\newblock URL \url{https://aclanthology.org/2023.acl-long.317}.

\bibitem[Bietti(2021)]{bietti_approximation_2021}
Bietti, A.
\newblock Approximation and {Learning} with {Deep} {Convolutional} {Models}: a
  {Kernel} {Perspective}.
\newblock October 2021.
\newblock URL \url{https://openreview.net/forum?id=lrocYB-0ST2}.

\bibitem[Bietti \& Bach(2021)Bietti and Bach]{bietti_deep_2021}
Bietti, A. and Bach, F.
\newblock Deep {Equals} {Shallow} for {ReLU} {Networks} in {Kernel} {Regimes},
  August 2021.
\newblock URL \url{http://arxiv.org/abs/2009.14397}.
\newblock arXiv:2009.14397 [cs, stat].

\bibitem[Bietti et~al.(2021)Bietti, Venturi, and Bruna]{bietti_sample_2021}
Bietti, A., Venturi, L., and Bruna, J.
\newblock On the {Sample} {Complexity} of {Learning} under {Invariance} and
  {Geometric} {Stability}, November 2021.
\newblock URL \url{http://arxiv.org/abs/2106.07148}.
\newblock arXiv:2106.07148 [cs, stat].

\bibitem[Bommasani et~al.(2021)Bommasani, Hudson, Adeli, Altman, Arora, von
  Arx, Bernstein, Bohg, Bosselut, Brunskill, Brynjolfsson, Buch, Card,
  Castellon, Chatterji, Chen, Creel, Davis, Demszky, Donahue, Doumbouya,
  Durmus, Ermon, Etchemendy, Ethayarajh, Fei-Fei, Finn, Gale, Gillespie, Goel,
  Goodman, Grossman, Guha, Hashimoto, Henderson, Hewitt, Ho, Hong, Hsu, Huang,
  Icard, Jain, Jurafsky, Kalluri, Karamcheti, Keeling, Khani, Khattab, Koh,
  Krass, Krishna, Kuditipudi, Kumar, Ladhak, Lee, Lee, Leskovec, Levent, Li,
  Li, Ma, Malik, Manning, Mirchandani, Mitchell, Munyikwa, Nair, Narayan,
  Narayanan, Newman, Nie, Niebles, Nilforoshan, Nyarko, Ogut, Orr,
  Papadimitriou, Park, Piech, Portelance, Potts, Raghunathan, Reich, Ren, Rong,
  Roohani, Ruiz, Ryan, R'e, Sadigh, Sagawa, Santhanam, Shih, Srinivasan,
  Tamkin, Taori, Thomas, Tram{\`e}r, Wang, Wang, Wu, Wu, Wu, Xie, Yasunaga,
  You, Zaharia, Zhang, Zhang, Zhang, Zhang, Zheng, Zhou, and
  Liang]{Bommasani2021FoundationModels}
Bommasani, R., Hudson, D.~A., Adeli, E., Altman, R., Arora, S., von Arx, S.,
  Bernstein, M.~S., Bohg, J., Bosselut, A., Brunskill, E., Brynjolfsson, E.,
  Buch, S., Card, D., Castellon, R., Chatterji, N.~S., Chen, A.~S., Creel,
  K.~A., Davis, J., Demszky, D., Donahue, C., Doumbouya, M., Durmus, E., Ermon,
  S., Etchemendy, J., Ethayarajh, K., Fei-Fei, L., Finn, C., Gale, T.,
  Gillespie, L.~E., Goel, K., Goodman, N.~D., Grossman, S., Guha, N.,
  Hashimoto, T., Henderson, P., Hewitt, J., Ho, D.~E., Hong, J., Hsu, K.,
  Huang, J., Icard, T.~F., Jain, S., Jurafsky, D., Kalluri, P., Karamcheti, S.,
  Keeling, G., Khani, F., Khattab, O., Koh, P.~W., Krass, M.~S., Krishna, R.,
  Kuditipudi, R., Kumar, A., Ladhak, F., Lee, M., Lee, T., Leskovec, J.,
  Levent, I., Li, X.~L., Li, X., Ma, T., Malik, A., Manning, C.~D.,
  Mirchandani, S.~P., Mitchell, E., Munyikwa, Z., Nair, S., Narayan, A.,
  Narayanan, D., Newman, B., Nie, A., Niebles, J.~C., Nilforoshan, H., Nyarko,
  J.~F., Ogut, G., Orr, L., Papadimitriou, I., Park, J.~S., Piech, C.,
  Portelance, E., Potts, C., Raghunathan, A., Reich, R., Ren, H., Rong, F.,
  Roohani, Y.~H., Ruiz, C., Ryan, J., R'e, C., Sadigh, D., Sagawa, S.,
  Santhanam, K., Shih, A., Srinivasan, K.~P., Tamkin, A., Taori, R., Thomas,
  A.~W., Tram{\`e}r, F., Wang, R.~E., Wang, W., Wu, B., Wu, J., Wu, Y., Xie,
  S.~M., Yasunaga, M., You, J., Zaharia, M.~A., Zhang, M., Zhang, T., Zhang,
  X., Zhang, Y., Zheng, L., Zhou, K., and Liang, P.
\newblock On the opportunities and risks of foundation models.
\newblock \emph{ArXiv}, 2021.
\newblock URL \url{https://crfm.stanford.edu/assets/report.pdf}.

\bibitem[Bordelon \& Pehlevan(2023)Bordelon and
  Pehlevan]{bordelon_self-consistent_2023}
Bordelon, B. and Pehlevan, C.
\newblock Self-consistent dynamical field theory of kernel evolution in wide
  neural networks*.
\newblock \emph{Journal of Statistical Mechanics: Theory and Experiment},
  2023\penalty0 (11):\penalty0 114009, November 2023.
\newblock ISSN 1742-5468.
\newblock \doi{10.1088/1742-5468/ad01b0}.
\newblock URL \url{https://dx.doi.org/10.1088/1742-5468/ad01b0}.
\newblock Publisher: IOP Publishing.

\bibitem[Bordelon et~al.(2020)Bordelon, Canatar, and
  Pehlevan]{bordelon_spectrum_2020}
Bordelon, B., Canatar, A., and Pehlevan, C.
\newblock Spectrum {Dependent} {Learning} {Curves} in {Kernel} {Regression} and
  {Wide} {Neural} {Networks}.
\newblock In \emph{Proceedings of the 37th {International} {Conference} on
  {Machine} {Learning}}, pp.\  1024--1034. PMLR, November 2020.
\newblock URL \url{https://proceedings.mlr.press/v119/bordelon20a.html}.
\newblock ISSN: 2640-3498.

\bibitem[Brown et~al.(2020)Brown, Mann, Ryder, Subbiah, Kaplan, Dhariwal,
  Neelakantan, Shyam, Sastry, Askell, Agarwal, Herbert-Voss, Krueger, Henighan,
  Child, Ramesh, Ziegler, Wu, Winter, Hesse, Chen, Sigler, Litwin, Gray, Chess,
  Clark, Berner, McCandlish, Radford, Sutskever, and
  Amodei]{brown_language_2020}
Brown, T., Mann, B., Ryder, N., Subbiah, M., Kaplan, J.~D., Dhariwal, P.,
  Neelakantan, A., Shyam, P., Sastry, G., Askell, A., Agarwal, S.,
  Herbert-Voss, A., Krueger, G., Henighan, T., Child, R., Ramesh, A., Ziegler,
  D., Wu, J., Winter, C., Hesse, C., Chen, M., Sigler, E., Litwin, M., Gray,
  S., Chess, B., Clark, J., Berner, C., McCandlish, S., Radford, A., Sutskever,
  I., and Amodei, D.
\newblock Language {Models} are {Few}-{Shot} {Learners}.
\newblock In \emph{Advances in {Neural} {Information} {Processing} {Systems}},
  volume~33, pp.\  1877--1901. Curran Associates, Inc., 2020.
\newblock URL
  \url{https://papers.nips.cc/paper/2020/hash/1457c0d6bfcb4967418bfb8ac142f64a-Abstract.html}.

\bibitem[Cagnetta et~al.(2023)Cagnetta, Favero, and Wyart]{cagnetta_what_2023}
Cagnetta, F., Favero, A., and Wyart, M.
\newblock What {Can} {Be} {Learnt} {With} {Wide} {Convolutional} {Neural}
  {Networks}?
\newblock In \emph{Proceedings of the 40th {International} {Conference} on
  {Machine} {Learning}}, pp.\  3347--3379. PMLR, July 2023.
\newblock URL \url{https://proceedings.mlr.press/v202/cagnetta23a.html}.
\newblock ISSN: 2640-3498.

\bibitem[Canatar et~al.(2021{\natexlab{a}})Canatar, Bordelon, and
  Pehlevan]{canatar_out--distribution_2021}
Canatar, A., Bordelon, B., and Pehlevan, C.
\newblock Out-of-{Distribution} {Generalization} in {Kernel} {Regression}.
\newblock In \emph{Advances in {Neural} {Information} {Processing} {Systems}},
  November 2021{\natexlab{a}}.
\newblock URL \url{https://openreview.net/forum?id=-h6Ldc0MO-}.

\bibitem[Canatar et~al.(2021{\natexlab{b}})Canatar, Bordelon, and
  Pehlevan]{canatar_spectral_2021}
Canatar, A., Bordelon, B., and Pehlevan, C.
\newblock Spectral bias and task-model alignment explain generalization in
  kernel regression and infinitely wide neural networks.
\newblock \emph{Nature Communications}, 12\penalty0 (1):\penalty0 2914, May
  2021{\natexlab{b}}.
\newblock ISSN 2041-1723.
\newblock \doi{10.1038/s41467-021-23103-1}.
\newblock URL \url{https://www.nature.com/articles/s41467-021-23103-1}.
\newblock Number: 1 Publisher: Nature Publishing Group.

\bibitem[Cao et~al.(2021)Cao, Fang, Wu, Zhou, and Gu]{cao_towards_2021}
Cao, Y., Fang, Z., Wu, Y., Zhou, D.-X., and Gu, Q.
\newblock Towards {Understanding} the {Spectral} {Bias} of {Deep} {Learning}.
\newblock volume~3, pp.\  2205--2211, August 2021.
\newblock \doi{10.24963/ijcai.2021/304}.
\newblock URL \url{https://www.ijcai.org/proceedings/2021/304}.
\newblock ISSN: 1045-0823.

\bibitem[Chen et~al.(2020)Chen, Radford, Child, Wu, Jun, Luan, and
  Sutskever]{chen_generative_2020}
Chen, M., Radford, A., Child, R., Wu, J., Jun, H., Luan, D., and Sutskever, I.
\newblock Generative {Pretraining} {From} {Pixels}.
\newblock In \emph{Proceedings of the 37th {International} {Conference} on
  {Machine} {Learning}}, pp.\  1691--1703. PMLR, November 2020.
\newblock URL \url{https://proceedings.mlr.press/v119/chen20s.html}.
\newblock ISSN: 2640-3498.

\bibitem[Cohen et~al.(2021)Cohen, Malka, and Ringel]{cohen_learning_2021}
Cohen, O., Malka, O., and Ringel, Z.
\newblock Learning curves for overparametrized deep neural networks: {A} field
  theory perspective.
\newblock \emph{Physical Review Research}, 3\penalty0 (2):\penalty0 023034,
  April 2021.
\newblock \doi{10.1103/PhysRevResearch.3.023034}.
\newblock URL \url{https://link.aps.org/doi/10.1103/PhysRevResearch.3.023034}.
\newblock Publisher: American Physical Society.

\bibitem[Cohen \& Welling(2016)Cohen and Welling]{cohen_group_2016}
Cohen, T.~S. and Welling, M.
\newblock Group {Equivariant} {Convolutional} {Networks}, June 2016.
\newblock URL \url{http://arxiv.org/abs/1602.07576}.
\newblock arXiv:1602.07576 [cs, stat].

\bibitem[Cowsik et~al.(2024)Cowsik, Nebabu, Qi, and
  Ganguli]{cowsik_geometric_2024}
Cowsik, A., Nebabu, T., Qi, X.-L., and Ganguli, S.
\newblock Geometric {Dynamics} of {Signal} {Propagation} {Predict}
  {Trainability} of {Transformers}, March 2024.
\newblock URL \url{http://arxiv.org/abs/2403.02579}.
\newblock arXiv:2403.02579 [cond-mat].

\bibitem[Dosovitskiy et~al.(2021)Dosovitskiy, Beyer, Kolesnikov, Weissenborn,
  Zhai, Unterthiner, Dehghani, Minderer, Heigold, Gelly, Uszkoreit, and
  Houlsby]{dosovitskiy_image_2021}
Dosovitskiy, A., Beyer, L., Kolesnikov, A., Weissenborn, D., Zhai, X.,
  Unterthiner, T., Dehghani, M., Minderer, M., Heigold, G., Gelly, S.,
  Uszkoreit, J., and Houlsby, N.
\newblock An {Image} is {Worth} 16x16 {Words}: {Transformers} for {Image}
  {Recognition} at {Scale}, June 2021.
\newblock URL \url{http://arxiv.org/abs/2010.11929}.
\newblock arXiv:2010.11929 [cs].

\bibitem[Dyer \& Gur-Ari(2019)Dyer and Gur-Ari]{dyer_asymptotics_2019}
Dyer, E. and Gur-Ari, G.
\newblock Asymptotics of {Wide} {Networks} from {Feynman} {Diagrams}.
\newblock In \emph{International {Conference} on {Learning} {Representations}},
  September 2019.
\newblock URL \url{https://openreview.net/forum?id=S1gFvANKDS}.

\bibitem[EuropeanCommission(2021)]{europeancommission_proposal_2021}
EuropeanCommission.
\newblock Proposal for a {REGULATION} {OF} {THE} {EUROPEAN} {PARLIAMENT} {AND}
  {OF} {THE} {COUNCIL} {LAYING} {DOWN} {HARMONISED} {RULES} {ON} {ARTIFICIAL}
  {INTELLIGENCE} ({ARTIFICIAL} {INTELLIGENCE} {ACT}) {AND} {AMENDING} {CERTAIN}
  {UNION} {LEGISLATIVE} {ACTS}, 2021.
\newblock URL
  \url{https://eur-lex.europa.eu/legal-content/EN/TXT/?uri=celex%3A52021PC0206}.

\bibitem[Fu et~al.(2023)Fu, Guo, Bai, and Mei]{fu_what_2023}
Fu, H., Guo, T., Bai, Y., and Mei, S.
\newblock What can a {Single} {Attention} {Layer} {Learn}? {A} {Study}
  {Through} the {Random} {Features} {Lens}.
\newblock \emph{Advances in Neural Information Processing Systems},
  36:\penalty0 11912--11951, December 2023.

\bibitem[Fulton \& Harris(2004)Fulton and Harris]{fulton_representation_2004}
Fulton, W. and Harris, J.
\newblock \emph{Representation {Theory}: {A} {First} {Course}}, volume 129 of
  \emph{Graduate {Texts} in {Mathematics}}.
\newblock Springer, New York, NY, 2004.
\newblock ISBN 978-3-540-00539-1 978-1-4612-0979-9.
\newblock \doi{10.1007/978-1-4612-0979-9}.
\newblock URL \url{http://link.springer.com/10.1007/978-1-4612-0979-9}.

\bibitem[Garg et~al.(2022)Garg, Tsipras, Liang, and Valiant]{garg_what_2022}
Garg, S., Tsipras, D., Liang, P., and Valiant, G.
\newblock What {Can} {Transformers} {Learn} {In}-{Context}? {A} {Case} {Study}
  of {Simple} {Function} {Classes}.
\newblock In \emph{Advances in {Neural} {Information} {Processing} {Systems}},
  May 2022.
\newblock URL \url{https://openreview.net/forum?id=flNZJ2eOet}.

\bibitem[Geifman et~al.(2022)Geifman, Galun, Jacobs, and
  Ronen]{geifman_spectral_2022}
Geifman, A., Galun, M., Jacobs, D., and Ronen, B.
\newblock On the {Spectral} {Bias} of {Convolutional} {Neural} {Tangent} and
  {Gaussian} {Process} {Kernels}.
\newblock \emph{Advances in Neural Information Processing Systems},
  35:\penalty0 11253--11265, December 2022.
\newblock URL
  \url{https://proceedings.neurips.cc/paper_files/paper/2022/hash/48fd58527b29c5c0ef2cae43065636e6-Abstract-Conference.html}.

\bibitem[GeminiTeam(2023)]{geminiteam_gemini_2023}
GeminiTeam.
\newblock Gemini: {A} {Family} of {Highly} {Capable} {Multimodal} {Models},
  December 2023.
\newblock URL \url{http://arxiv.org/abs/2312.11805}.
\newblock arXiv:2312.11805 [cs].

\bibitem[Goyal \& Bengio(2022)Goyal and Bengio]{goyal_inductive_2022}
Goyal, A. and Bengio, Y.
\newblock Inductive biases for deep learning of higher-level cognition.
\newblock \emph{Proceedings of the Royal Society A: Mathematical, Physical and
  Engineering Sciences}, 478\penalty0 (2266):\penalty0 20210068, October 2022.
\newblock \doi{10.1098/rspa.2021.0068}.
\newblock URL
  \url{https://royalsocietypublishing.org/doi/full/10.1098/rspa.2021.0068}.
\newblock Publisher: Royal Society.

\bibitem[Henighan et~al.(2020)Henighan, Kaplan, Katz, Chen, Hesse, Jackson,
  Jun, Brown, Dhariwal, Gray, Hallacy, Mann, Radford, Ramesh, Ryder, Ziegler,
  Schulman, Amodei, and McCandlish]{henighan_scaling_2020}
Henighan, T., Kaplan, J., Katz, M., Chen, M., Hesse, C., Jackson, J., Jun, H.,
  Brown, T.~B., Dhariwal, P., Gray, S., Hallacy, C., Mann, B., Radford, A.,
  Ramesh, A., Ryder, N., Ziegler, D.~M., Schulman, J., Amodei, D., and
  McCandlish, S.
\newblock Scaling {Laws} for {Autoregressive} {Generative} {Modeling}, November
  2020.
\newblock URL \url{http://arxiv.org/abs/2010.14701}.
\newblock arXiv:2010.14701 [cs].

\bibitem[Hron et~al.(2020)Hron, Bahri, Sohl-Dickstein, and
  Novak]{hron_infinite_2020}
Hron, J., Bahri, Y., Sohl-Dickstein, J., and Novak, R.
\newblock Infinite attention: {NNGP} and {NTK} for deep attention networks,
  June 2020.
\newblock URL \url{http://arxiv.org/abs/2006.10540}.
\newblock arXiv:2006.10540 [cs, stat].

\bibitem[Jacot et~al.(2018)Jacot, Gabriel, and Hongler]{jacot_neural_2018}
Jacot, A., Gabriel, F., and Hongler, C.
\newblock Neural {Tangent} {Kernel}: {Convergence} and {Generalization} in
  {Neural} {Networks}.
\newblock In \emph{Advances in {Neural} {Information} {Processing} {Systems}},
  volume~31. Curran Associates, Inc., 2018.

\bibitem[Jiang et~al.(2024)Jiang, Sablayrolles, Roux, Mensch, Savary, Bamford,
  Chaplot, Casas, Hanna, Bressand, Lengyel, Bour, Lample, Lavaud, Saulnier,
  Lachaux, Stock, Subramanian, Yang, Antoniak, Scao, Gervet, Lavril, Wang,
  Lacroix, and Sayed]{jiang_mixtral_2024}
Jiang, A.~Q., Sablayrolles, A., Roux, A., Mensch, A., Savary, B., Bamford, C.,
  Chaplot, D.~S., Casas, D. d.~l., Hanna, E.~B., Bressand, F., Lengyel, G.,
  Bour, G., Lample, G., Lavaud, L.~R., Saulnier, L., Lachaux, M.-A., Stock, P.,
  Subramanian, S., Yang, S., Antoniak, S., Scao, T.~L., Gervet, T., Lavril, T.,
  Wang, T., Lacroix, T., and Sayed, W.~E.
\newblock Mixtral of {Experts}, January 2024.
\newblock URL \url{http://arxiv.org/abs/2401.04088}.
\newblock arXiv:2401.04088 [cs].

\bibitem[Kaplan et~al.(2020)Kaplan, McCandlish, Henighan, Brown, Chess, Child,
  Gray, Radford, Wu, and Amodei]{kaplan_scaling_2020}
Kaplan, J., McCandlish, S., Henighan, T., Brown, T.~B., Chess, B., Child, R.,
  Gray, S., Radford, A., Wu, J., and Amodei, D.
\newblock Scaling {Laws} for {Neural} {Language} {Models}, January 2020.
\newblock URL \url{http://arxiv.org/abs/2001.08361}.
\newblock arXiv:2001.08361 [cs, stat].

\bibitem[Kim et~al.(2021)Kim, Oh, and Hong]{kim_transformers_2021}
Kim, J., Oh, S., and Hong, S.
\newblock Transformers {Generalize} {DeepSets} and {Can} be {Extended} to
  {Graphs} \& {Hypergraphs}.
\newblock In \emph{Advances in {Neural} {Information} {Processing} {Systems}},
  volume~34, pp.\  28016--28028. Curran Associates, Inc., 2021.

\bibitem[{König}(1986)]{konig_eigenvalue_1986}
{König}, H.
\newblock \emph{Eigenvalue {Distribution} of {Compact} {Operators}}, volume~16
  of \emph{Operator {Theory}: {Advances} and {Applications}}.
\newblock {Birkhäuser}, Basel, 1986.
\newblock ISBN 978-3-0348-6280-6 978-3-0348-6278-3.
\newblock \doi{10.1007/978-3-0348-6278-3}.
\newblock URL \url{http://link.springer.com/10.1007/978-3-0348-6278-3}.

\bibitem[Lavie \& Ringel(2024)Lavie and Ringel]{future_work}
Lavie, I. and Ringel, Z.
\newblock Symmetric kernels with non-symmetric data: A data-agnostic
  learnability bound.
\newblock 2024.

\bibitem[Lee et~al.(2018)Lee, Bahri, Novak, Schoenholz, Pennington, and
  Sohl-Dickstein]{lee_deep_2018}
Lee, J., Bahri, Y., Novak, R., Schoenholz, S.~S., Pennington, J., and
  Sohl-Dickstein, J.
\newblock Deep {Neural} {Networks} as {Gaussian} {Processes}.
\newblock February 2018.
\newblock URL \url{https://openreview.net/forum?id=B1EA-M-0Z}.

\bibitem[Lee et~al.(2019)Lee, Lee, Kim, Kosiorek, Choi, and Teh]{lee_set_2019}
Lee, J., Lee, Y., Kim, J., Kosiorek, A., Choi, S., and Teh, Y.~W.
\newblock Set {Transformer}: {A} {Framework} for {Attention}-based
  {Permutation}-{Invariant} {Neural} {Networks}.
\newblock In \emph{Proceedings of the 36th {International} {Conference} on
  {Machine} {Learning}}, pp.\  3744--3753. PMLR, May 2019.
\newblock URL \url{https://proceedings.mlr.press/v97/lee19d.html}.
\newblock ISSN: 2640-3498.

\bibitem[Lee et~al.(2020)Lee, Schoenholz, Pennington, Adlam, Xiao, Novak, and
  Sohl-Dickstein]{lee_finite_2020}
Lee, J., Schoenholz, S.~S., Pennington, J., Adlam, B., Xiao, L., Novak, R., and
  Sohl-Dickstein, J.
\newblock Finite {Versus} {Infinite} {Neural} {Networks}: an {Empirical}
  {Study}, September 2020.
\newblock URL \url{http://arxiv.org/abs/2007.15801}.
\newblock arXiv:2007.15801 [cs, stat].

\bibitem[Lewkowycz et~al.(2020)Lewkowycz, Bahri, Dyer, Sohl-Dickstein, and
  Gur-Ari]{lewkowycz_large_2020}
Lewkowycz, A., Bahri, Y., Dyer, E., Sohl-Dickstein, J., and Gur-Ari, G.
\newblock The large learning rate phase of deep learning: the catapult
  mechanism, March 2020.
\newblock URL \url{http://arxiv.org/abs/2003.02218}.
\newblock arXiv:2003.02218 [cs, stat].

\bibitem[Li \& Sompolinsky(2021)Li and Sompolinsky]{li_statistical_2021}
Li, Q. and Sompolinsky, H.
\newblock Statistical {Mechanics} of {Deep} {Linear} {Neural} {Networks}: {The}
  {Backpropagating} {Kernel} {Renormalization}.
\newblock \emph{Physical Review X}, 11\penalty0 (3):\penalty0 031059, September
  2021.
\newblock \doi{10.1103/PhysRevX.11.031059}.
\newblock URL \url{https://link.aps.org/doi/10.1103/PhysRevX.11.031059}.
\newblock Publisher: American Physical Society.

\bibitem[Liu et~al.(2021)Liu, Ziyin, and Ueda]{liu_noise_2021}
Liu, K., Ziyin, L., and Ueda, M.
\newblock Noise and {Fluctuation} of {Finite} {Learning} {Rate} {Stochastic}
  {Gradient} {Descent}, June 2021.
\newblock URL \url{http://arxiv.org/abs/2012.03636}.
\newblock arXiv:2012.03636 [cs, stat].

\bibitem[Liu et~al.(2022{\natexlab{a}})Liu, Kitouni, Nolte, Michaud, Tegmark,
  and Williams]{liu_towards_2022}
Liu, Z., Kitouni, O., Nolte, N.~S., Michaud, E., Tegmark, M., and Williams, M.
\newblock Towards {Understanding} {Grokking}: {An} {Effective} {Theory} of
  {Representation} {Learning}.
\newblock \emph{Advances in Neural Information Processing Systems},
  35:\penalty0 34651--34663, December 2022{\natexlab{a}}.
\newblock URL
  \url{https://proceedings.neurips.cc/paper_files/paper/2022/hash/dfc310e81992d2e4cedc09ac47eff13e-Abstract-Conference.html}.

\bibitem[Liu et~al.(2022{\natexlab{b}})Liu, Michaud, and
  Tegmark]{liu_omnigrok_2022}
Liu, Z., Michaud, E.~J., and Tegmark, M.
\newblock Omnigrok: {Grokking} {Beyond} {Algorithmic} {Data}.
\newblock In \emph{The {Eleventh} {International} {Conference} on {Learning}
  {Representations}}, September 2022{\natexlab{b}}.
\newblock URL \url{https://openreview.net/forum?id=zDiHoIWa0q1}.

\bibitem[Mandt et~al.(2018)Mandt, Hoffman, and Blei]{mandt_stochastic_2018}
Mandt, S., Hoffman, M.~D., and Blei, D.~M.
\newblock Stochastic {Gradient} {Descent} as {Approximate} {Bayesian}
  {Inference}, January 2018.
\newblock URL \url{http://arxiv.org/abs/1704.04289}.
\newblock arXiv:1704.04289 [cs, stat].

\bibitem[Merity et~al.(2016)Merity, Xiong, Bradbury, and
  Socher]{merity_pointer_2016}
Merity, S., Xiong, C., Bradbury, J., and Socher, R.
\newblock Pointer {Sentinel} {Mixture} {Models}.
\newblock In \emph{International {Conference} on {Learning} {Representations}},
  November 2016.
\newblock URL \url{https://openreview.net/forum?id=Byj72udxe}.

\bibitem[Min et~al.(2022)Min, Chen, Bian, Xu, Zhao, Huang, Zhao, Huang,
  Ananiadou, and Rong]{min_transformer_2022}
Min, E., Chen, R., Bian, Y., Xu, T., Zhao, K., Huang, W., Zhao, P., Huang, J.,
  Ananiadou, S., and Rong, Y.
\newblock Transformer for {Graphs}: {An} {Overview} from {Architecture}
  {Perspective}, February 2022.
\newblock URL \url{http://arxiv.org/abs/2202.08455}.
\newblock arXiv:2202.08455 [cs].

\bibitem[Mingard et~al.(2020)Mingard, Valle-Pérez, Skalse, and
  Louis]{mingard_is_2020}
Mingard, C., Valle-Pérez, G., Skalse, J., and Louis, A.~A.
\newblock Is {SGD} a {Bayesian} sampler? {Well}, almost, October 2020.
\newblock URL \url{http://arxiv.org/abs/2006.15191}.
\newblock arXiv:2006.15191 [cs, stat].

\bibitem[Mohtashami et~al.(2023)Mohtashami, Jaggi, and
  Stich]{mohtashami_special_2023}
Mohtashami, A., Jaggi, M., and Stich, S.
\newblock Special {Properties} of {Gradient} {Descent} with {Large} {Learning}
  {Rates}, February 2023.
\newblock URL \url{http://arxiv.org/abs/2205.15142}.
\newblock arXiv:2205.15142 [cs, math].

\bibitem[Nam et~al.(2022)Nam, Abdool, Maxfield, and
  McClelland]{nam_achieving_2022}
Nam, A.~J., Abdool, M., Maxfield, T., and McClelland, J.~L.
\newblock Achieving and {Understanding} {Out}-of-{Distribution}
  {Generalization} in {Systematic} {Reasoning} in {Small}-{Scale}
  {Transformers}, December 2022.
\newblock URL \url{http://arxiv.org/abs/2210.03275}.
\newblock arXiv:2210.03275 [cs].

\bibitem[Nanda et~al.(2023)Nanda, Chan, Lieberum, Smith, and
  Steinhardt]{nanda_progress_2023}
Nanda, N., Chan, L., Lieberum, T., Smith, J., and Steinhardt, J.
\newblock Progress measures for grokking via mechanistic interpretability,
  October 2023.
\newblock URL \url{http://arxiv.org/abs/2301.05217}.
\newblock arXiv:2301.05217 [cs].

\bibitem[Naveh et~al.(2021)Naveh, Ben~David, Sompolinsky, and
  Ringel]{naveh_predicting_2021}
Naveh, G., Ben~David, O., Sompolinsky, H., and Ringel, Z.
\newblock Predicting the outputs of finite deep neural networks trained with
  noisy gradients.
\newblock \emph{Physical Review E}, 104\penalty0 (6):\penalty0 064301, December
  2021.
\newblock \doi{10.1103/PhysRevE.104.064301}.
\newblock URL \url{https://link.aps.org/doi/10.1103/PhysRevE.104.064301}.
\newblock Publisher: American Physical Society.

\bibitem[Neal(1993)]{neal_probabilistic_1993}
Neal, R.~M.
\newblock Probabilistic {Inference} {Using} {Markov} {Chain} {Monte} {Carlo}
  {Methods}, 1993.

\bibitem[Neal(1996)]{neal_priors_1996}
Neal, R.~M.
\newblock Priors for {Infinite} {Networks}.
\newblock In Neal, R.~M. (ed.), \emph{Bayesian {Learning} for {Neural}
  {Networks}}, Lecture {Notes} in {Statistics}, pp.\  29--53. Springer, New
  York, NY, 1996.
\newblock ISBN 978-1-4612-0745-0.
\newblock \doi{10.1007/978-1-4612-0745-0_2}.
\newblock URL \url{https://doi.org/10.1007/978-1-4612-0745-0_2}.

\bibitem[Noci et~al.(2022)Noci, Anagnostidis, Biggio, Orvieto, Singh, and
  Lucchi]{noci_signal_2022}
Noci, L., Anagnostidis, S., Biggio, L., Orvieto, A., Singh, S.~P., and Lucchi,
  A.
\newblock Signal {Propagation} in {Transformers}: {Theoretical} {Perspectives}
  and the {Role} of {Rank} {Collapse}.
\newblock \emph{Advances in Neural Information Processing Systems},
  35:\penalty0 27198--27211, December 2022.

\bibitem[Novak et~al.(2018)Novak, Xiao, Bahri, Lee, Yang, Hron, Abolafia,
  Pennington, and Sohl-dickstein]{novak_bayesian_2018}
Novak, R., Xiao, L., Bahri, Y., Lee, J., Yang, G., Hron, J., Abolafia, D.~A.,
  Pennington, J., and Sohl-dickstein, J.
\newblock Bayesian {Deep} {Convolutional} {Networks} with {Many} {Channels} are
  {Gaussian} {Processes}.
\newblock September 2018.
\newblock URL \url{https://openreview.net/forum?id=B1g30j0qF7}.

\bibitem[Olsson et~al.(2022)Olsson, Elhage, Nanda, Joseph, DasSarma, Henighan,
  Mann, Askell, Bai, Chen, Conerly, Drain, Ganguli, Hatfield-Dodds, Hernandez,
  Johnston, Jones, Kernion, Lovitt, Ndousse, Amodei, Brown, Clark, Kaplan,
  McCandlish, and Olah]{olsson_-context_2022}
Olsson, C., Elhage, N., Nanda, N., Joseph, N., DasSarma, N., Henighan, T.,
  Mann, B., Askell, A., Bai, Y., Chen, A., Conerly, T., Drain, D., Ganguli, D.,
  Hatfield-Dodds, Z., Hernandez, D., Johnston, S., Jones, A., Kernion, J.,
  Lovitt, L., Ndousse, K., Amodei, D., Brown, T., Clark, J., Kaplan, J.,
  McCandlish, S., and Olah, C.
\newblock In-context {Learning} and {Induction} {Heads}, September 2022.
\newblock URL \url{http://arxiv.org/abs/2209.11895}.
\newblock arXiv:2209.11895 [cs].

\bibitem[OpenAI(2023)]{openai_gpt-4_2023}
OpenAI.
\newblock {GPT}-4 {Technical} {Report}, March 2023.
\newblock URL \url{http://arxiv.org/abs/2303.08774}.
\newblock arXiv:2303.08774 [cs].

\bibitem[Pacelli et~al.(2023)Pacelli, Ariosto, Pastore, Ginelli, Gherardi, and
  Rotondo]{pacelli_statistical_2023}
Pacelli, R., Ariosto, S., Pastore, M., Ginelli, F., Gherardi, M., and Rotondo,
  P.
\newblock A statistical mechanics framework for {Bayesian} deep neural networks
  beyond the infinite-width limit.
\newblock \emph{Nature Machine Intelligence}, 5\penalty0 (12):\penalty0
  1497--1507, December 2023.
\newblock ISSN 2522-5839.
\newblock \doi{10.1038/s42256-023-00767-6}.
\newblock URL \url{https://www.nature.com/articles/s42256-023-00767-6}.
\newblock Number: 12 Publisher: Nature Publishing Group.

\bibitem[Potters \& Bouchaud(2020)Potters and Bouchaud]{potters_first_2020}
Potters, M. and Bouchaud, J.-P.
\newblock \emph{A {First} {Course} in {Random} {Matrix} {Theory}: for
  {Physicists}, {Engineers} and {Data} {Scientists}}.
\newblock Cambridge University Press, Cambridge, 2020.
\newblock ISBN 978-1-108-48808-2.
\newblock \doi{10.1017/9781108768900}.
\newblock URL
  \url{https://www.cambridge.org/core/books/first-course-in-random-matrix-theory/2292A554A9BB9E2A4697C35BCE920304}.

\bibitem[Power et~al.(2022)Power, Burda, Edwards, Babuschkin, and
  Misra]{power_grokking_2022}
Power, A., Burda, Y., Edwards, H., Babuschkin, I., and Misra, V.
\newblock Grokking: {Generalization} {Beyond} {Overfitting} on {Small}
  {Algorithmic} {Datasets}, January 2022.
\newblock URL \url{http://arxiv.org/abs/2201.02177}.
\newblock arXiv:2201.02177 [cs].

\bibitem[Rahaman et~al.(2019)Rahaman, Baratin, Arpit, Draxler, Lin, Hamprecht,
  Bengio, and Courville]{rahaman_spectral_2019}
Rahaman, N., Baratin, A., Arpit, D., Draxler, F., Lin, M., Hamprecht, F.,
  Bengio, Y., and Courville, A.
\newblock On the {Spectral} {Bias} of {Neural} {Networks}.
\newblock In \emph{Proceedings of the 36th {International} {Conference} on
  {Machine} {Learning}}, pp.\  5301--5310. PMLR, May 2019.
\newblock URL \url{https://proceedings.mlr.press/v97/rahaman19a.html}.
\newblock ISSN: 2640-3498.

\bibitem[Rubin et~al.(2023)Rubin, Seroussi, and Ringel]{rubin_droplets_2023}
Rubin, N., Seroussi, I., and Ringel, Z.
\newblock Droplets of {Good} {Representations}: {Grokking} as a {First} {Order}
  {Phase} {Transition} in {Two} {Layer} {Networks}, October 2023.
\newblock URL \url{http://arxiv.org/abs/2310.03789}.
\newblock arXiv:2310.03789 [cond-mat, stat].

\bibitem[Sagan(2001)]{sagan_symmetric_2001}
Sagan, B.~E.
\newblock \emph{The {Symmetric} {Group}}, volume 203 of \emph{Graduate {Texts}
  in {Mathematics}}.
\newblock Springer, New York, NY, 2001.
\newblock ISBN 978-1-4419-2869-6 978-1-4757-6804-6.
\newblock \doi{10.1007/978-1-4757-6804-6}.
\newblock URL \url{http://link.springer.com/10.1007/978-1-4757-6804-6}.

\bibitem[Scetbon \& Harchaoui(2021)Scetbon and
  Harchaoui]{scetbon_spectral_2021}
Scetbon, M. and Harchaoui, Z.
\newblock A {Spectral} {Analysis} of {Dot}-product {Kernels}.
\newblock In \emph{Proceedings of {The} 24th {International} {Conference} on
  {Artificial} {Intelligence} and {Statistics}}, pp.\  3394--3402. PMLR, March
  2021.
\newblock URL \url{https://proceedings.mlr.press/v130/scetbon21b.html}.
\newblock ISSN: 2640-3498.

\bibitem[{Schölkopf} et~al.(2001){Schölkopf}, Herbrich, and
  Smola]{scholkopf_generalized_2001}
{Schölkopf}, B., Herbrich, R., and Smola, A.~J.
\newblock A {Generalized} {Representer} {Theorem}.
\newblock In Helmbold, D. and Williamson, B. (eds.), \emph{Computational
  {Learning} {Theory}}, pp.\  416--426, Berlin, Heidelberg, 2001. Springer.
\newblock ISBN 978-3-540-44581-4.
\newblock \doi{10.1007/3-540-44581-1_27}.

\bibitem[Segadlo et~al.(2022)Segadlo, Epping, Meegen, Dahmen, Krämer, and
  Helias]{segadlo_unified_2022}
Segadlo, K., Epping, B., Meegen, A.~v., Dahmen, D., Krämer, M., and Helias, M.
\newblock Unified field theoretical approach to deep and recurrent neuronal
  networks.
\newblock \emph{Journal of Statistical Mechanics: Theory and Experiment},
  2022\penalty0 (10):\penalty0 103401, October 2022.
\newblock ISSN 1742-5468.
\newblock \doi{10.1088/1742-5468/ac8e57}.
\newblock URL \url{https://dx.doi.org/10.1088/1742-5468/ac8e57}.
\newblock Publisher: IOP Publishing and SISSA.

\bibitem[Seroussi et~al.(2023)Seroussi, Naveh, and
  Ringel]{seroussi_separation_2023}
Seroussi, I., Naveh, G., and Ringel, Z.
\newblock Separation of scales and a thermodynamic description of feature
  learning in some {CNNs}.
\newblock \emph{Nature Communications}, 14\penalty0 (1):\penalty0 908, February
  2023.
\newblock ISSN 2041-1723.
\newblock \doi{10.1038/s41467-023-36361-y}.
\newblock URL \url{https://www.nature.com/articles/s41467-023-36361-y}.
\newblock Number: 1 Publisher: Nature Publishing Group.

\bibitem[Shankar et~al.(2020)Shankar, Fang, Guo, Fridovich-Keil, Schmidt,
  Ragan-Kelley, and Recht]{shankar_neural_2020}
Shankar, V., Fang, A., Guo, W., Fridovich-Keil, S., Schmidt, L., Ragan-Kelley,
  J., and Recht, B.
\newblock Neural {Kernels} {Without} {Tangents}, March 2020.
\newblock URL \url{http://arxiv.org/abs/2003.02237}.
\newblock arXiv:2003.02237 [cs, stat].

\bibitem[Silverman(1984)]{silverman_spline_1984}
Silverman, B.~W.
\newblock Spline {Smoothing}: {The} {Equivalent} {Variable} {Kernel} {Method}.
\newblock \emph{The Annals of Statistics}, 12\penalty0 (3):\penalty0 898--916,
  1984.
\newblock ISSN 0090-5364.
\newblock URL \url{https://www.jstor.org/stable/2240968}.
\newblock Publisher: Institute of Mathematical Statistics.

\bibitem[Simon et~al.(2023)Simon, Dickens, Karkada, and
  Deweese]{simon_eigenlearning_2023}
Simon, J.~B., Dickens, M., Karkada, D., and Deweese, M.
\newblock The {Eigenlearning} {Framework}: {A} {Conservation} {Law}
  {Perspective} on {Kernel} {Ridge} {Regression} and {Wide} {Neural}
  {Networks}.
\newblock \emph{Transactions on Machine Learning Research}, February 2023.
\newblock ISSN 2835-8856.
\newblock URL \url{https://openreview.net/forum?id=FDbQGCAViI}.

\bibitem[Sollich \& Williams(2004)Sollich and Williams]{sollich_using_2004}
Sollich, P. and Williams, C.
\newblock Using the {Equivalent} {Kernel} to {Understand} {Gaussian} {Process}
  {Regression}.
\newblock In \emph{Advances in {Neural} {Information} {Processing} {Systems}},
  volume~17. MIT Press, 2004.

\bibitem[Tahmasebi \& Jegelka(2023)Tahmasebi and Jegelka]{tahmasebi_exact_2023}
Tahmasebi, B. and Jegelka, S.
\newblock The {Exact} {Sample} {Complexity} {Gain} from {Invariances} for
  {Kernel} {Regression}.
\newblock \emph{Advances in Neural Information Processing Systems},
  36:\penalty0 55616--55646, December 2023.

\bibitem[Tung(1985)]{tung_group_1985}
Tung, W.-K.
\newblock \emph{Group {Theory} in {Physics}: {An} {Introduction} to {Symmetry}
  {Principles}, {Group} {Representations}, and {Special} {Functions} in
  {Classical} and {Quantum} {Physics}}.
\newblock WORLD SCIENTIFIC, August 1985.
\newblock ISBN 978-9971-966-57-7 978-981-238-498-0.
\newblock \doi{10.1142/0097}.
\newblock URL \url{http://www.worldscientific.com/worldscibooks/10.1142/0097}.

\bibitem[Von~Oswald et~al.(2023)Von~Oswald, Niklasson, Randazzo, Sacramento,
  Mordvintsev, Zhmoginov, and Vladymyrov]{von_oswald_transformers_2023}
Von~Oswald, J., Niklasson, E., Randazzo, E., Sacramento, J., Mordvintsev, A.,
  Zhmoginov, A., and Vladymyrov, M.
\newblock Transformers {Learn} {In}-{Context} by {Gradient} {Descent}.
\newblock In \emph{Proceedings of the 40th {International} {Conference} on
  {Machine} {Learning}}, pp.\  35151--35174. PMLR, July 2023.
\newblock URL \url{https://proceedings.mlr.press/v202/von-oswald23a.html}.
\newblock ISSN: 2640-3498.

\bibitem[Welling \& Teh(2011)Welling and Teh]{welling_bayesian_2011}
Welling, M. and Teh, Y.~W.
\newblock Bayesian learning via stochastic gradient langevin dynamics.
\newblock In \emph{Proceedings of the 28th {International} {Conference} on
  {International} {Conference} on {Machine} {Learning}}, {ICML}'11, pp.\
  681--688, Madison, WI, USA, June 2011. Omnipress.
\newblock ISBN 978-1-4503-0619-5.

\bibitem[Wen et~al.(2023)Wen, Li, Liu, and Risteski]{wen_transformers_2023}
Wen, K., Li, Y., Liu, B., and Risteski, A.
\newblock Transformers are uninterpretable with myopic methods: a case study
  with bounded {Dyck} grammars, December 2023.
\newblock URL \url{http://arxiv.org/abs/2312.01429}.
\newblock arXiv:2312.01429 [cs, stat].

\bibitem[Wolf et~al.(2020)Wolf, Debut, Sanh, Chaumond, Delangue, Moi, Cistac,
  Rault, Louf, Funtowicz, Davison, Shleifer, von Platen, Ma, Jernite, Plu, Xu,
  Scao, Gugger, Drame, Lhoest, and Rush]{wolf_huggingfaces_2020}
Wolf, T., Debut, L., Sanh, V., Chaumond, J., Delangue, C., Moi, A., Cistac, P.,
  Rault, T., Louf, R., Funtowicz, M., Davison, J., Shleifer, S., von Platen,
  P., Ma, C., Jernite, Y., Plu, J., Xu, C., Scao, T.~L., Gugger, S., Drame, M.,
  Lhoest, Q., and Rush, A.~M.
\newblock {HuggingFace}'s {Transformers}: {State}-of-the-art {Natural}
  {Language} {Processing}, July 2020.
\newblock URL \url{http://arxiv.org/abs/1910.03771}.
\newblock arXiv:1910.03771 [cs].

\bibitem[Xiao(2022)]{xiao_eigenspace_2022}
Xiao, L.
\newblock Eigenspace {Restructuring}: {A} {Principle} of {Space} and
  {Frequency} in {Neural} {Networks}.
\newblock In \emph{Proceedings of {Thirty} {Fifth} {Conference} on {Learning}
  {Theory}}, pp.\  4888--4944. PMLR, June 2022.
\newblock URL \url{https://proceedings.mlr.press/v178/xiao22a.html}.
\newblock ISSN: 2640-3498.

\bibitem[Xie et~al.(2021)Xie, Raghunathan, Liang, and Ma]{xie_explanation_2021}
Xie, S.~M., Raghunathan, A., Liang, P., and Ma, T.
\newblock An {Explanation} of {In}-context {Learning} as {Implicit} {Bayesian}
  {Inference}.
\newblock In \emph{International {Conference} on {Learning} {Representations}},
  October 2021.
\newblock URL \url{https://openreview.net/forum?id=RdJVFCHjUMI}.

\bibitem[Yang \& Salman(2020)Yang and Salman]{yang_fine-grained_2020}
Yang, G. and Salman, H.
\newblock A {Fine}-{Grained} {Spectral} {Perspective} on {Neural} {Networks},
  April 2020.
\newblock URL \url{http://arxiv.org/abs/1907.10599}.
\newblock arXiv:1907.10599 [cs, stat].

\end{thebibliography}
\bibliographystyle{icml2024}

\newpage
\appendix
\numberwithin{equation}{section}
\onecolumn
\section{Introduction to Key Concepts in Representation Theory for Eigenvalue Problems}
\label{appendix:rep_fourier_intro}
Symmetries can greatly simplify the above eigenvalue problem. Let $G$ be a symmetry group, we say the eigenvalue problem possesses this symmetry provided \begin{align}
    \forall g \in G, ~ &
    k(T_g~ x, T_g~ y) = k(x, y) \\ \nonumber
    & \pdata(T_g~ x)=\pdata(x)
\end{align} where the linear transformations ($T_g$) are some faithful representation of $G$ (i.e. $T_g T_{g'}=T_{gg'}$ and $T_{g} T_{g'} = Id$ iff $g g'$ is the identity element of $G$). 

As a concrete example and to make contact with the terminology in the main text, consider the case where $x \in {\rm R}^2$ which we express in polar coordinates $x=(r_x\cos(\theta_x),r_x\sin(\theta_x))$, and $p(x)$ effectively discretizes $\theta$ and fixes $r$ (i.e. $p(x)=\delta(r_x-1)N^{-1}\sum_{j=1}^N \delta(\theta_x-2\pi j/N)$). Let $K(x,y)=||x-y||$, $G=Z_N$ given by the rotation of $x$ in units of $2\pi/N$, and $T_g$'s given by the corresponding $2 \times 2$ rotation matrices on $x$. 

Next we utilize $G$ to find the spectrum of $K$ w.r.t. $p(x)$. To this end, we consider the space on which $\hat{K}$ acts--- the vector space of functions of $x$ ($f(x)$) with the distance induced by $p(x)$. This space is $N$ dimensional and spanned by $[f(x_1),...,f(x_N)] \equiv \vec{f}$. The linear action of $T_g$ on $x$ induces a linear action on function space (equivalently on $\vec{f}$) via $\hat{T}_g \cdot f(x) = f(T_g x)$. Symmetry under $G$, as defined above, implies that $\hat{T}_g$'s all commute with $\hat{K}$. Consequently eigenspaces of $\hat{K}$ are invariant under all $\hat{T}_g$'s.


The above guides us to look for the minimal vector spaces which are invariant under all $\hat{T}_g$'s. These are known as irreducible representations (irreps). The group $Z_N$ is known to have $N$ distinct irreducible representations of dimension 1 labelled by $k \in \{ 2\pi/N,4\pi/N,...,2\pi\}$. The corresponding invariant spaces are simply the discrete Fourier mode vectors $\vec{v}_k=[e^{2\pi i k/N},e^{4\pi i k/N},...,1]$. It can be checked that all $\hat{T}_g$'s leave each of these spaces/vectors invariant. This implies $\hat{K}$ is diagonal on the $\vec{v}_k$ basis. Allowing more complicated radial dependence, say by taking $p(x)$ with $\delta(r-1)$ replaced by $\frac{1}{2}[\delta(r-1) + \delta(r-2)]$, the resulting blocks of $\hat{K}$ associated with each irrep would be $2\times 2$. Equivalently stated each block would contain the irrep at multiplicity $2$. Furthermore, for non-abelian $G$ (e.g. augmenting $Z_N$ with reflections), irreps of dimension larger than $1$ generally appear.

\section{A Gentle Introduction to The Use of Symmetry in Kernel Learning and The Symmetric Group}
\label{appendix:sym-group-irreps}
Spectral properties of kernels with respect to the data measure, provide a detailed description of the implicit bias of infinitely wide neural networks. However, diagonalizing a generic kernel operator on a generic measure is challenging. For fully connected networks and rotation symmetric datasets, this difficulty is largely lifted. In fact for uniform data on the hypersphere closed-form expressions for the spectrum and eigenfunctions exist \cite{cohen_learning_2021,canatar_spectral_2021},
the latter being hyperspherical harmonics. These results follow directly from studying the representation theory of the orthogonal group acting on multivariate polynomials. 

For transformer models like the ones introduced above, the analog task is to find representations of the symmetric group acting on multivariate polynomials. Below we provide several concrete examples of such representations, flesh out their implications on spectral bias, and provide a road map for deriving higher representations.  

As a starting point consider a kernel $K(x,y)$ where $x,y \in {\mathcal R}^d$ and some generic dataset measure $p(x)$. Let $S_d$ denote the symmetric group (the group of all possible permutations) on $1,2,...,d$ where an element $\sigma \in S_d$ acts on $x$ as $[\sigma x]_i = x_{\sigma(i)}$ (i.e. the natural action). Assuming $K(x,y)=K(\sigma x,\sigma y)$ and $p(x)=p(\sigma x)$ we wish to solve the following eigenvalue problem 
\begin{align}
\int dy p(y) K(x,y) \varphi_{\lambda}(y) &= \lambda \varphi_{\lambda}(x) 
\end{align}
to simplify the problem, let us assume that $k(x,y)$ contains powers of $x$ and $y$ only up to some finite degree $q$. In that case, any $\varphi_{\lambda}(x)$ with non-zero $\lambda$ must be at most a $q$'th order multivariate polynomial. 

To proceed with finding the spectrum and eigenfunctions, we first address the question of what are the irreducible representations of the symmetric group acting on finite degree polynomials. Irreducible representations  (irreps) of the symmetric group are labelled by partitions of $d$ which we denote by $(d_1,d_2,...,d_k)$ such that $d_1 \geq d_2 \geq ... \geq d_k$ and $\sum_k d_k = d$. These partitions are in one-to-one correspondence with Young Diagrams wherein one simply draws a row of boxes of length $d_1$, followed by a left aligned row of boxes of length $d_2$ etc... 

Conveniently, there is a direct way of constructing an irrep from its Young diagram (see \citet{fulton_representation_2004}). As shown in theorem \ref{thm:irreps of multilinear}, particularly relevant here are Young diagrams of the form $(n-k,k)$. Considering those, the first step is finding all standard Young Tableaux associated with the Young diagram. Standard Young Tableaux are assignments of integers between $1..d$, with no repetitions, to the boxes of the Young Diagram such that all columns and rows are of increasing order. For instance, for the case $(d-2,2)$ and $d=6$ these would be 
\begin{align}
&\begin{ytableau}
       1 & 3 & 5 & 6  \\
       2 & 4 
\end{ytableau} \,\,\,\,\
\begin{ytableau}
       1 & 3 & 4 & 6  \\
       2 & 5 
\end{ytableau} \,\,\,\,\
\begin{ytableau}
       1 & 3 & 4 & 5  \\
       2 & 6 
\end{ytableau} \,\,\,\,\
 \begin{ytableau}
       1 & 2 & 5 & 6  \\
       3 & 4 
\end{ytableau} \,\,\,\,\
\begin{ytableau}
       1 & 2 & 4 & 6  \\
       3 & 5 
\end{ytableau}  \\ \nonumber
& \begin{ytableau}
       1 & 2 & 4 & 5  \\
       3 & 6 
\end{ytableau}  \,\,\,\,\
\begin{ytableau}
       1 & 2 & 3 & 6  \\
       4 & 5 
\end{ytableau} \,\,\,\,\
\begin{ytableau}
       1 & 2 & 3 & 5  \\
       4 & 6 
\end{ytableau} \,\,\,\,\
\begin{ytableau}
       1 & 2 & 3 & 4  \\
       5 & 6 
\end{ytableau} \,\,\,\,\
\end{align}
An important observation here, true for any $(d-k,k)$, is that the lower row completely determines the upper one. Indeed the upper row must consist of all integers besides those in the lower row, arranged in strictly increasing order. We may thus denote such tableaux by their set of lower row integers $i_1,i_2,..,i_k$ (although some combinations may be disallowed). We next associated a monomial of the form $x_{i_1} x_{i_2}...x_{i_k}$ with each such standard Young Tableaux\footnote{Similar to the construction of Specht modules from Young tabloids\cite{fulton_representation_2004}.} \footnote{In the next appendix, where we prove theorem \ref{thm:irreps of multilinear} we take a different approach for the construction of the irreps of the Symmetric group. Here we effectively directly associate monomials with Young Tabloids, while in the next appendix, we use the Young symmetrizers as projectors to irrep spaces without the need for such a less formal, yet more intuitive, association between tabloids and monomials.}. To proceed with the construction we further consider the group of column permutations $C \subset S_d$ wherein we only allow switching of pairs along columns. We then construct the following polynomial element from the monomial
\begin{align}
M^{1}(x)_{i_1..i_k} &= \sum_{\sigma \in C} \sign(\sigma) x_{\sigma i_1}..x_{\sigma i_k}
\end{align}
it then follows (see appendix \ref{appendix:irreps_theorem}) that these $k$'th degree polynomials span the irreps $(d-k,k)$, where the action of $S_d$ amounts to its natural action on the indices $x$. Notably this basis is typically not an orthonormal one. Furthermore, the same representation may appear with any power of $x_{i}$, namely $M^{(m)}=\sum_{\sigma \in C} \sign(\sigma) x^m_{\sigma i_1}..x^m_{\sigma i_k}$,  $m \in \mathbb{N}$, however for discrete measures some of these may collapse onto one another or to the trivial representation. For instance if $x_i \in \{+1,-1\}$,$M^{(2m)}$ is just a constant and $M^{(2m+1)}=M^{(1)}$. 

One notable example of a $(d-k,k)$ representation is the standard representation $(d-1,1)$ equivalent to the natural action on
\begin{align}
{\rm Span} \{ x_i - x_0 \}_{i=1}^{d}
\end{align}
this representation is also equivalent to considering the discrete Fourier modes 
\begin{align}
\varphi_k(x) &= \sum_j e^{i 2\pi k j / d} x_j \,\,\,\ k \in \{ 1,2,...,d-1\}
\end{align}
but omitting $\varphi_{k=0}(x)$ (the trivial representation). The different $k$ numbers, via $e^{2\pi i k/d}$, can also be understood as one-dimensional-irreps of the cyclic group ($Z_n \subset S$). 

Another relevant irrep is the trivial one, corresponding to symmetric (multivariate) polynomials. These are spanned by the Schur polynomials which are again in one-to-one correspondence with Young Diagrams, via however a different association than the one above. Up to an order of, say order $3$, these are spanned by $1$, $\sum_i x_i$,$\sum_{i=j} x_i x_j$, $\sum_{i \neq j}x_i x_j,\sum_{i=j=k} x_i x_j x_k$, $\sum_{i \neq j=k} x_i x_j x_k$, $\sum_{i \neq j \neq k}x_i x_j x_k$. 

Another low dimensional representation is the sign representation of the symmetric group, associated with alternating polynomials (polynomials which are anti-symmetric with respect to exchanging any two variables). All such polynomials are of degree higher than that of the Vandermonde polynomial ($\pi_{1\leq i < j \leq d,n-d}(x_i - x_j)$), thus having a degree higher than $n-1+n-2+...+0=n(n-1)/2$. Due to their high order they would not appear for any $q<d$. We conjecture that these would be exponentially suppressed in $d$ for any NNGP or NTK kernel.    

The above irreps and their associations with polynomials, facilitate the construction of low order polynomial representations. For instance, let us assume that $x_i \in \{+1,-1\}$ and consider all possible polynomials up to second order. These are spanned by three trivial representations (i.e. $(d)$ partition/Young-Diagram) 
\begin{align}
1,\sum_{i=1}^d x_i,\sum_{1 \leq i < j \leq d,n-d} x_i x_j 
\end{align}
two $d-1$ dimension standard representations ($(d-1,1)$)
\begin{align}
\varphi_k(x) \,\,\,\,\ k \in \{1..d-1\} \\ \nonumber
\left(\sum_{i=1}^d x_i\right)\varphi_k(x) \,\,\,\,\ k \in \{1..d-1\}
\end{align}
and one $(d-1)(d-2)/2 - 1$ dimension ($(d-2,2)$) representation spanned by 
\begin{align}
\varphi_{ij}(x) = x_i x_j - x_0 x_j - x_i x_b + x_0 x_b \,\,\,\  b = \min[ \{ k\}_{k=1}^{d} \backslash \{i,j\}] , 1 < i < j \neq 3
\end{align}
Given a measure ($p(x)$) which respects the symmetry, any two polynomials associated with distinct representation would be orthogonal. However, their normalization and the orthogonality relations within the same representations would vary based on the measure. 

Turning to the spectrum, it then follows from standard representation theory arguments that a kernel with $q=2$ has 6 generally distinct eigenvalues. Three generally-non-degenerate eigenvalues are associated with linear combinations of the 3 trivial representations. Two, generally distinct sets, of $d-1$-degenerate eigenvalues associated with the two linear combinations of the two standard representations. Last, one $(d-1)(d-2)/2-1$ degenerate eigenvalue associated with the $(d-2,2)$ representations. 

Finally, we note that the eigenfunctions associated with the two standard representations can mix in a limited manner. Following the assignment of $k$ numbers (or equivalently eigenvalues with respect to the subgroup of $S$ consisting of cyclic permutations of the indices), each basis element we used is also an irrep of the cyclic group. Hence two different values of $k$ cannot be mixed. In addition, other elements in the permutation group are capable of shifting between these $k$ values, hence the linear combinations are constant as a function of $k$. As the eigenfunctions associated with one of the $d-1$-degenerate eigenvalue can be written as $a \varphi_k + b (\sum_i x_i) \varphi_k$ where $a,b$ are $k$ independent coefficients. The corresponding eigenfunction associated with the other $d-1$-degenerate eigenvalues is simply the orthogonal one.  

\section{Decomposition of multilinear polynomials to irreps of the symmetric group}
\label{appendix:irreps_theorem}

\begin{definition}[Partition] A partition of $n$ is an ordered set of positive integers $\lambda=(\lambda_1 , \lambda_2 , ... ,\lambda_m)$ such that $\left \{ \lambda_i \right \}_{i=1}^m \subset \mathbb{N}$, $\sum_{i=1}^m \lambda_i = n$ and $\lambda_1 \geq \lambda_2 \geq ... \geq \lambda_m \geq 1$.
\end{definition}

\begin{theorem}
    Irreps of the symmetric group of $n$ symbols $S_n$ are uniquely labeled by partitions  of $n$~\citep{fulton_representation_2004}
\end{theorem}

\begin{definition}[Young Diagram]
A Young diagram $\Theta_\lambda$ of a partition $\lambda$ of $n$ is a diagram where one draws a row of $\lambda_i$ boxes for each element in lambda starting with $\lambda_1$, with each subsequent element below it. For example given the partition $\lambda=(3,2,1)$ the Young diagram is
\begin{equation}
\Theta_\lambda = \ydiagram{3,2,1}
\end{equation}
\end{definition}

\begin{definition}[Young Tableau]
A Young Tableau $\Theta^p_\lambda$ associated with a Young diagram $\Theta_\lambda$ with $n$ boxes is a filling where each box is filled with an integer$1,...,n$ with no repetitions (definition vary, here we follow \cite{sagan_symmetric_2001}).
For example some of the Young Tableaux associated with the Young diagram from the previous example are:
\begin{equation}
\Theta_\lambda^C =
\begin{ytableau}
    1 & 2 & 3 \\
    4 & 5 \\
    6
\end{ytableau} ~, ~ 
\Theta_\lambda^a =
\begin{ytableau}
    2 & 5 & 3 \\
    6 & 1 \\
    4
\end{ytableau} ~, ~ 
\Theta_\lambda^b =
\begin{ytableau}
    1 & 4 & 5 \\
    2 & 6 \\
    3
\end{ytableau}
\label{eq:ytableau examples}
\end{equation}
\end{definition}

\begin{definition}[Standard Young Tableau] A standard Young tableau is a Young tableau where the rows and columns increase to the right and to the bottom respectively (Again definitions vary, here we follow~\cite{sagan_symmetric_2001}). For example, $\Theta^b_\lambda$ and $\Theta^C_\lambda$ in~\eqref{eq:ytableau examples} are standard Young tableaux but $\Theta^a_\lambda$ is not. 
\end{definition}

\begin{definition}[Canonical Young Tableau] A canonical Young tableau $\Theta^C_\lambda$ is a standard Young tableau where the numbers $1,...,\lambda_1$ appear in the first row, the numbers $\lambda_1+1,...,\lambda_2$ appear is the second row and so on.  For example, the $\Theta^C_\lambda$ in~\eqref{eq:ytableau examples} is the canonical Young tableau.
\end{definition}

\begin{definition}[Rows and columns subgroups]
\label{rows and columns subgroups}
Given a Young tableau $\Theta^p_\lambda$ of partition $\lambda$ and assignment $p$, we define the rows subgroup $\mathcal{R}^p_\lambda$ which leave invariant the (unordered) sets of numbers appearing in the same row of $\Theta^p_\lambda$. Similarly, we define columns subgroup $\mathcal{C}^p_\lambda$ which leave invariant the (unordered) sets of numbers appearing in the same column of $\Theta^p_\lambda$.
\end{definition}

\begin{definition}[Permutation action on the multilinear polynomials]
Let $\mathcal{T}$ be linear representations of the the symmetric group $S_n$ on the multilinear polynomials, such that the permutation acts naturally on the variables indices. E.g. let $\sigma \in S_n$ be a permutation, and let $P(x_1,...,x_n)=x_1 x_2$ be a multilinear polynomial, then $\mathcal{T}(s) P = x_{\sigma(1)} x_{\sigma(2)}$.    

Define the groups of row and column actions on the multilinear polynomials 
\begin{equation}
R^p_\lambda = \left\{ \mathcal{T} (\sigma) | \sigma \in \mathcal{R}^p_\lambda \right\}
, \quad
C^p_\lambda = \left\{ \mathcal{T} (\sigma) | \sigma \in \mathcal{C}^p_\lambda \right\}
\end{equation}
\end{definition}

\begin{definition} [Row symmetrizer, column anti-symmetrizer and young symmetrizer]
Define the row symmetrizer, column anti-symmetrizer and Young symmetrizer linear operators:
        
\begin{align}
\hat{R}_\lambda^p &= \sum_{r\in R_\lambda^p} r \\
\hat{C}_\lambda^p &= \sum_{c\in C_\lambda^p} {\rm sign}(c) c \\
\hat{Y}^p_\lambda &= \hat{C}^p_\lambda \hat{R}^p_\lambda \end{align}    
\end{definition}

\begin{theorem}
\label{theorem:Young projectors}
    Young symmetrizers associated with standard Young tableaux are projectors to irrep spaces of the symmetric group~\citep{fulton_representation_2004}
\end{theorem}

\begin{lemma}
\label{lemma:transposition_vanish}
If there exists a transposition $t^* \in C_\lambda^p$ that leaves a monomial $M$ unchanged, $M$ vanishes under the action of the column anti-symmetrizer - $$
\exists t^* \in C_{\lambda}^{p}~~{\rm s.t.}~~t^* M=M \rightarrow\hat{C}_{\lambda}^{p} M=0.
$$
\begin{proof}
	Let $\Theta^p_\lambda$ be a standard Young tableau of a partition $\lambda$.
	Let $\hat{C}_\lambda^p$ be the column anti-symmetrizer associated with $\Theta^p_\lambda$.
    Let $M(x_1,x_2,...,x_n)$ be a multilinear monomial in the variables $x_1,x_2,...,x_n$.
	Let $t^* \in C_\lambda^p$ be a transposition such that $t^*M =M$.
	A transposition is an involution, that means, it is a bijection from the group to itself and $t^* t^* = e$, where $e$ is the identity element. 
	Right multiplication with $t^*$ maps any element $c_i \in C_\lambda^p$ from the column group to $c_j = c_i t^*$ such that 
 \begin{align}
 {\rm sign} (c_i) c_i M = {\rm sign} (c_i t^* t^*) c_i t^* t^* M  = {\rm sign} (c_j t^*) c_j t^* M = - {\rm sign} (c_j) c_j M.
 \end{align}
	
     We have constructed a unique pairing between each $c_i \in C_\lambda^p$ and $c_j \in C_\lambda^p$ such that $c_i \neq c_j$ and ${\rm sign} (c_i) c_i M = - {\rm sign} (c_j) c_j M$ that is 
        $$\forall c_i \in C_\lambda^p~~ \exists!  c_j \in C_\lambda^p~~{\rm s.t.}~~
        c_i \neq c_j \land{\rm sign} (c_i) c_i M = - {\rm sign} (c_j) c_j M.$$
    That means the terms in the sum cancel in pairs $\hat{C}_\lambda^p  M = \sum_{c\in C_\lambda^p} {\rm sign}(c) c M = 0$.
\end{proof}
\end{lemma}

\begin{lemma}
\label{lemma:rows<=2}
All multilinear monomials in $n$ variables, vanish when acted upon with a column anti-symmetrizer that corresponds to a Young tableau with more than $2$ rows
\begin{proof}
	Let $M(x_1,x_2,...,x_n)$ be a multilinear monomial in the variables $x_1,x_2,...,x_n$.
	Let $\Theta^p_\lambda$ be a standard Young tableau of a partition $\lambda$ that has more than 2 rows.
	The first column in $\Theta^p_\lambda$  gives raise to at least $3$ transpositions $(ab),(bc),(ac)$.
	Since each variable must either appear in $M(x_1,x_2,...,x_n)$ to a single power or zeroth power, out of the $3$ variables $x_a,x_b,x_c$ at least two must appear to the same power.
	Because the product of our variables is not ordered, at least one of the $3$ transpositions leaves $M(x_1,x_2,...,x_n)$ unchanged.
	Applying lemma \ref{lemma:transposition_vanish}, $M(x_1,x_2,...,x_n)$ must vanish under the action.
\end{proof}
\end{lemma}

\begin{lemma}
\label{lemma:k<=d}
All multilinear monomials of degree $d$ in $n$ variables, vanish when acted upon with a column anti-symmetrizer associated with a partition $(n-k,k)$ for $k>\min\{d,n-d\}$.
\begin{proof}
	for $k>\{d,n-d\}$ there exists a column transposition $(ab)\in C_\lambda^p$ where both $x_a,x_b$ appear in the monomial to zeroth power, therefore the transposition $(ab)$ leaves it unchanged.
	Applying lemma \ref{lemma:transposition_vanish}, the monomial must vanish under the action.
\end{proof}
\end{lemma}
	
\begin{remark}
The multilinear monomial can be thought of as picking specific boxes in the Young tableau, one can then permute inside the rows, writing down the numbers that appear in the chosen boxes as the indices in the monomial. Finally one can act with the column permutations, while adding their signs, on the monomials found by the rows actions. Summing up all terms gives the result of acting with the Young symmetrizer on the monomial.
The necessary conditions above for $\hat{C}_\lambda^p M \neq 0$ translate to being able to pick $d$ boxes such that at most one box is picked in every column, and no column has more than one box unpicked in it.
\end{remark}

\begin{lemma}
\label{prop:k<=d_exists}
There exists a multilinear monomial of degree $d$ in $n$ variables, that does not vanish when acted upon with a Young symmetrizer associated with a partition $(n-k,k)$ for every $k$ such that $0\leq k \leq d, n-d$.
\begin{proof}
Let $M = \prod_{i=1}^{d} x_i$ 
be a multilinear monomial of degree $d$ in $n$ variables. Let $\Theta^C_{(n-k,k)}$ be the canonical Young tableau associated with the partition $(n-k,k)$ for $0\leq k \leq d,n-d$,
\begin{equation}
\Theta^C_{(n-k,k)} = 
\ytableausetup{boxsize=3em}
\begin{ytableau}
       1 & 2 & \none[\dots] & k  & \none[\dots] & \scriptstyle n-k \\
       \scriptstyle {n-k+1} & \scriptstyle {n-k+2} & \none[\dots] & n
\end{ytableau}
~~.
\end{equation}

We now verify $\hat{Y}^C_{(n-k,k)} M \neq 0$:

The row symmetrizer sums positive elements, therefore the sum cannot vanish 
\begin{equation}
P = \hat{R}^C_{(n-k,k)} M = \sum_{r \in R^C_{(n-k,k)}} r M \neq 0.
\end{equation}

Since $\{ x_i \}_{i=1}^{n}$ are independent variables all elements in the sum above are linearly independent (up to identical elements). We may conclude it is sufficient to show a single element doesn't vanish to prove $\hat{C}^C_{(n-k,k)} P$ doesn't vanish, since $\hat{C}^C_{(n-k,k)}$ includes the trivial element.
In particular, we will show that for $r=e$ the summand $rM=M$ does not vanish under the action of the column symmetrizer.

The column symmetrizer $\hat{C}^C_{(n-k,k)}$ is a sum of closed, independent, column transpositions and their products. All non-trivial transpositions, when acting on $M$ specifically, create linearly independent elements, therefore the sum of such transpositions acting on $M$ cannot vanish.

We may conclude $\hat{C}^C_{(n-k,k)} P$ includes at least one non vanishing term (that is $M$) and therefore $\hat{Y}^C_{(n-k,k)} M \neq 0$.
\end{proof}
\end{lemma}

\begin{definition}[Hook Length]
    The hook length $h_\lambda (i,j)$ of a box, where $i$ ($j$) denotes the row (column) of the box in the Young diagram $\Theta_\lambda$, is the number of boxes to the right of the $i,j$'th box in the $i$'th row, plus the number of boxes below the box in the $j$'th column plus one.
\end{definition}

\begin{lemma}
\label{lemma:irepps_dim}
The dimension of an irrep associated with a partition $(n-k,k)$ is ${\rm dim}_\lambda=\frac{n!}{k! \frac{(n-k+1)!}{n-2k+1}}$.
\begin{proof}
using the hook length formula~\citep{fulton_representation_2004} 
$${\rm dim}_\lambda=\frac{n !}{\prod_{i,j\in\lambda} h_\lambda(i, j)} .$$
The product in the denominator equals
\begin{align}
\prod_{i,j\in\lambda} h_\lambda(i, j) &= \underbrace{(n-2k)!}_\text{upper row with nothing below} \underbrace{k!}_\text{lower row} \underbrace{\frac{(n-k+1)!}{(n-2k+1)!}}_\text{upper row with boxes below} =k!\frac{(n-k+1)!}{n-2k+1}=\binom{n+1}{k} \frac{(n+1)!}{n-2k+1}.
\end{align}
Resulting in $${\rm dim}_\lambda= \frac{n!}{k! \frac{(n-k+1)!}{n-2k+1}} \sim n^{k}.$$
\end{proof}
\end{lemma}

\irrepsOfMultilinear*
\begin{proof}
    Let $\Theta^p_\lambda$ be a standard Young tableau of a partition $\lambda$.
	Let $\hat{R}_\lambda^p, \hat{C}_\lambda^p, \hat{Y}^p_\lambda$ be the row symmetrizer, column anti-symmetrizer and Young symmetrizer (respectively) of the $\Theta^p_\lambda$.

    Let $\{M_n^d\}$ be the set of all multilinear monomials in $n$ variables of degree $d$.

    $\{M_n^d\}$ is a basis for the space of multilinear polynomials in $n$ variables of degree $d$. That means ${\rm Span} \{M_n^d\}$ is the space of multilinear polynomials in $n$ variables of degree $d$.

    ${\rm Span} \{M_n^d\}$ is closed under the action of $\hat{R}_\lambda^p$. 
    Therefore, if $\forall M \in \{M_n^d\},~ \hat{C}_\lambda^p M = 0$, 
    then $\forall P \in {\rm Span} \{M_n^d\},~\hat{Y}^p_\lambda P = 0$.
    
    Using lemmas \ref{lemma:rows<=2},\ref{lemma:k<=d} we see that all $P \in {\rm Span} \{M_n^d\}$ vanish under the action of the Young symmetrizers associated with a Young diagram with more than 2 rows or more than $\min\{d,n-d\}$ boxes on the second row.
    
    
    Based on lemma \ref{prop:k<=d_exists} and theorem \ref{theorem:Young projectors} each of the irreps $(n-k,k)$ $0 \leq k \leq d,n-d$ appears at least once in the decomposition of ${\rm Span} \{M_n^d\}$ into irreps of the symmetric group.

    ${\rm Span} \{M_n^d\}$ is $\binom{n}{d}$ dimensional.
    
    Summing the dimension of the irreps (lemma \ref{lemma:irepps_dim})  $$\sum_{k=0}^{\min\{d,n-d\}} \frac{n!}{k! \frac{(n-k+1)!}{n-2k+1}}=\binom{n}{d}$$Since the sum of dimensions of irreps equals the dimension of the space each irrep appears only once.
\end{proof} 



    

\section{Linear example}
\label{appendix:linear_example}
In section \ref{subsec:Expressibility} we have found the kernel can be written in be simplified down to a scalar form
\begin{equation}
\begin{aligned}
k(X,Y)= &\underbrace{\frac{1}{8}  \left( x^{L+1}_1 y^{L+1}_1  +  \left (1-x^{L+1}_1 \right) \left(1- y^{L+1}_1 \right)  \right)}_\mathfrak{A} \cdot
&\underbrace{\left[
\begin{aligned}
&\frac{1}{L^2} \sum_{a,b=1}^{L} \left(  x^{a}_1 y^{b}_1  +  \left (1-x^{a}_1 \right) \left(1- y^{b}_1 \right) \right) \\
+ &\frac{1}{L^2} \sum_{a=1}^{L} \left(  x^{a}_1 y^{a}_1  +  \left (1-x^{a}_1 \right) \left(1- y^{a}_1 \right)  \right)+ \frac{1}{L}  
\end{aligned}
\right]}_\mathfrak{B}
,
\end{aligned}
\label{eq:simplified-kernel}
\end{equation}
revealing the space of expressible function to be the zeroth and first-degree polynomials of the tokens in the context (part $\mathfrak{B}$) window multiplied by the last token(part $\mathfrak{A}$).
We followed up by applying the Thm.~\ref{thm:irreps of multilinear}, to reveal they are composed of standard and trivial irreps.

Turning to the space of the standard irrep, it can be further decomposed to one-dimensional irreps of the cyclic subgroup known as the Fourier modes, thereby acquiring eigenvectors of $\mathfrak{B}$.
Putting these together with the eigenvectors of $\mathfrak{A}$ $a(\vec{x}^{L+1}),~b(\vec{x}^{L+1})$ we find $2(L-2)$ eigenvectors of the kernel (given explicitly in \eqref{eq:k_vectors}).

The eigenvalues are all independent of $k \in \{ 1,2,...(L/2-1) \}$ since all the $k$ modes belong to the same irrep, and only differ by $O (1)$ factor from one another based on the difference between odd and even and the $a,b$ subspaces
\vspace{-0.2cm}
\begin{equation}
\lambda^{\rm odd}_{k,a}, \lambda^{\rm even}_{k,a}, \lambda^{\rm odd}_{k,b}, \lambda^{\rm even}_{k,b} \propto \frac{1}{L^2}
\end{equation}
full expressions are given by
\begin{equation}
\left\{
\begin{aligned}
\binom{{\varphi}^{\rm odd}_{k,a} (X)}{{\varphi}^{\rm odd}_{k,b} (X)} = \binom{\frac{x^{L+1}}{Z^{\rm odd}_{k,a}}}{\frac{1-x^{L+1}}{Z^{\rm odd}_{k,b}}} \sum_{s=1}^{L/2} e^{i \frac{ \pi k}{L/2} s} x^{2s-1}_1,
\quad
\binom{{\varphi}^{\rm even}_{k,a}(X)}{{\varphi}^{\rm even}_{k,b}(X)} = \binom{\frac{x^{L+1}}{Z^{\rm even}_{k,a}}}{\frac{1-x^{L+1}}{Z^{even}_{k,b}}} \sum_{s=1}^{L/2} e^{i \frac{ \pi k}{L/2} s} x^{2s}_1
\end{aligned}
\right\}_{k=1}^{L/2-1},
\label{eq:k_vectors}
\end{equation}
with $Z^{\rm odd}_{k,a},Z^{\rm odd}_{k,b},Z^{even}_{k,a},Z^{even}_{k,b}$ being appropriate normalization constants.

Following the same procedure we find the trivial representation is spanned by
\vspace{-0.2cm}
\begin{equation}
\tilde{\varphi}^{\rm odd}_0 (X) = \sum_{s=1}^{L/2} x^{2s-1}_1;~~ \tilde{\varphi}^{\rm even}_0 (X) = \sum_{s=1}^{L/2} x^{2s}_1;~~ \tilde{\varphi}_c(X) = 1.
\end{equation}
By a Gram–Schmidt like-process, we find a good basis for the space of permutation invariant functions ${\varphi}_{c,*},{\varphi}_{0,*}^{+},{\varphi}_{0,*}^{-}$ with $*=\{a,b\}$; the definitions are given in \eqref{eq:phi_base_def}. The diagonalization in the multiplicity spaces of the trivial irrep can now be carried out numerically or analytically in closed form as it can be written as two $3 \times 3$ matrices.

\subsection{Learnable target}
\label{subsec:Learnable target}
So far, the whole process has been task-independent, the last component required to predict the output of the NN is the projections of the target onto the eigenvectors, which depend on the target function and the training distribution. Since the task requires estimating a parameter not accessible to the network, the projections can never span the true target function, instead even as $N\to \infty$ the network will learn a different function which we dub the \emph{learnable target} given by $\sum_i g_i \varphi_i(x)$. We denote the projections by $g_{*}^{-},g_{*}^{+},g_{c,*},g_{k,*}^{{\rm odd}},g_{k,*}^{{\rm even}}$ for ${\varphi}_{0,*}^{-},{\varphi}_{0,*}^{+},{\varphi}_{c,*},{\varphi}_{k,*}^{{\rm odd}},{\varphi}_{k,*}^{{\rm even}}$ respectively, where $*=\{a,b\}$.
This projections depend on the parameters of the training distribution $p_{a},q_{a},w,L$.
Keeping only leading orders of $w,\frac{1}{L}$ we find $g_{k,*}^{{\rm odd}},g_{k,*}^{{\rm even}}$ vanish for all $k$, and $g_{c,*}$ are constants w.r.t $w,L$ while 
\begin{equation}
g_{*}^{+} =\frac{Lw^{2}\eta_{*}^{+}}{\sqrt{L^{2}w^{2}\rho_{*}^{+}+L\sigma_{*}^{+}}}, \quad
g_{*}^{-} =\frac{Lw^{2}\eta_{*}^{-}+\nu_{*}^{-}}{\sqrt{L^{2}w^{4}\rho_{*}^{-}+Lw^{2}\sigma_{*}^{-}+\xi_{*}^{-}}}, 
\label{eq:learning-target-coeff-short}
\end{equation}
the definitions of 
$\eta_{*}^{\star}$, $\nu_{*}^{\star}$, $\rho_{*}^{\star}$, $\sigma_{*}^{\star}$, $\xi_{*}^{\star}$, where $*=\{a,b\}$ and $\star=\{+,-\}$, are detailed in appendix \ref{appendix:full-expressions}.

\section{Out of distribution predictions under EK approximation}
\label{appendix:OOD-MSE}
\newcommand{\pTrain}{{p_{\rm train}}}
\newcommand{\Etrain}{\E_{X\sim\ptrain}}
\newcommand{\Edata}{\E_{X\sim\pdata}}
\newcommand{\ptest}{{p_{\rm test}}}
\newcommand{\Etest}{\E_{X\sim\ptest}}
\newcommand{\Eparam}{\E_{\Theta}}
We would like to compute the mean square loss (MSE) on an arbitrary data distribution $\pdata$ when training on a data distribution $\pTrain$ for which we know the EK predictor. 

Under EK approximation MSE loss can be computed by
\begin{equation}
\begin{aligned}
&\Edata\Eparam\left[\left(f_{\Theta}\left(X\right)-g\left(X\right)\right)^{2}\right]=
\Edata\Eparam\left[\left(f_{\Theta}\left(X\right)-g\left(X\right)\right)^{2}\right]= \\
&=\Edata\left[\Eparam\left[f_{\Theta}^{2}\left(X\right)\right]-2\Eparam\left[f_{\Theta}\left(X\right)\right]g\left(X\right)+g^{2}\left(X\right)\right]\approx\Edata\left[\Eparam\left[f_{\Theta}\left(X\right)\right]^{2}-2\Eparam\left[f_{\Theta}\left(X\right)\right]g\left(X\right)+g^{2}\left(X\right)\right]=
\\
&=\Edata\left[\left[\sum_{i}\frac{\lambda_{i}}{\lambda_{i}+\sigma^{2}/N}g_{i}\varphi_{i}\left(x\right)\right]^{2}-2\sum_{i}\frac{\lambda_{i}}{\lambda_{i}+\sigma^{2}/N}g_{i}\varphi_{i}\left(x\right)g\left(X\right)+g^{2}\left(X\right)\right],
\end{aligned}
\end{equation}

Where the approximation on the second line is dropping the EK variance
\begin{equation}
\Eparam\left[f_{\Theta}\left(X\right)\right]^{2}=\Eparam\left[f_{\Theta}\left(X\right)\right]^{2}+{\rm Var}\left[f_{\Theta}\left(X\right)\right]\approx\Eparam\left[f_{\Theta}\left(X\right)\right]^{2} .
\label{eq:MSE_EK_appendix}
\end{equation}
One can in fact calculate this quantity easily within the GP framework but we found the approximation to be good enough as is and chose to drop it for simplicity. This result is exact when using an ensemble, or when using the EK predictor itself as the predictor.

By the spectral theorem, are guaranteed $\phi_i(x)$ to be orthogonal w.r.t. an inner product defined with $\pTrain$ as a weighting function as defined in \eqref{eq:inner-product-def}. Therefore, choosing $\pdata=\pTrain$ simplifies \eqref{eq:MSE_EK_appendix} to

\begin{equation}
    =\sum_{i }\left(\frac{\lambda_{i}}{\lambda_{i}+\sigma^{2}/N}\right)^{2}g_{i}^{2}-2\sum_{i}\frac{\lambda_{i}}{\lambda_{i}+\sigma^{2}/N}g_{i}^{2}+\left\langle g,g\right\rangle _{X\sim\pTrain},
\end{equation}
as $\left\langle \varphi_{i},\varphi_{j}\right\rangle_\pTrain = \delta_{ij}$, with $\delta_{ij}$ being the Kronecker delta.

Now, if we wish to compute the loss under distributional shift all we have to do is take the expectation value in \eqref{eq:MSE_EK_appendix} w.r.t. a new distribution $\pdata=\ptest \neq \pTrain$
\begin{equation}
 \sum_{i}\sum_{j}\frac{\lambda_i}{\lambda_i+\sigma^{2}/N}\frac{\mu_j}{\mu_j+\sigma^{2}/N}g_{i}g_{j}\left\langle \varphi_{i},\varphi_{j}\right\rangle _{X\sim\ptest}-2\sum_{i}\frac{\lambda_i}{\lambda_i+\sigma^{2}/N}g_{i}\left\langle \varphi_{i},g\right\rangle _{X\sim\ptest}+\left\langle g,g\right\rangle _{X\sim\ptest}~.
\end{equation}
Notably, the eigenfunctions that were orthonormal under the inner product induced by the training distribution are no longer necessarily orthonormal under the test distribution, and the object $G_{ij}:=\left\langle \varphi_{i},\varphi_{j}\right\rangle _{X\sim\ptest}$ can be identified as their Gram matrix. Nevertheless, the elements of the Gram matrix can be evaluated as sums or integrals, analytically or numerically. In our case, we evaluated the elements of $G_{ij}$ for the full parameter space of $p_a,q_a,w$ analytically.

\section{Sub-leading corrections from $x^{L+1}$}
\label{appendix:corrections-xL+1}
The terms left out during the approximation are
\begin{equation}
\begin{aligned}
k^{\left(1\right)}\left(X,Y\right)=&\frac{1}{8L^{2}}\left(x^{L+1}y^{L+1}+\left(1-x^{L+1}\right)\left(1-y^{L+1}\right)\right)\cdot...
\\
...\cdot & \left[\sum_{a=1}^{L}x^{L+1}y^{a}+\sum_{a=1}^{L}\left(1-x^{L+1}\right)\left(1-y^{a}\right)+\sum_{a=1}^{L}x^{a}y^{L+1}+\sum_{a=1}^{L}\left(1-x^{a}\right)\left(1-y^{L+1}\right)+... \right. \\
...+ &  \left. 3x^{L+1}y^{L+1}+3\left(1-x^{L+1}\right)\left(1-y^{L+1}\right)+1\right]
\end{aligned}
\end{equation}
All the vectors ${\varphi}^{\rm odd}_{k,a} (X),~{\varphi}^{\rm even}_{k,a} (X)~,{\varphi}^{\rm odd}_{k,b} (X),~{\varphi}^{\rm even}_{k,b} (X)$ in the standard representation get no corrections at all as their matrix elements with all basis vectors vanish.

Moving on to the two $3\times3$ blocks of the trivial representation, ${\varphi}_{0,a}^+,{\varphi}_{0,a}^-$ (${\varphi}_{0,b}^+,{\varphi}_{0,b}^-$) can only have non-vanishing matrix elements with the ${\varphi}_{c,a}$ (${\varphi}_{c,b}$). These terms are at most $O\left(\frac{1}{L^{3}}\right)$; furthermore, they are second-order corrections in the eigenvalue perturbation and are therefore sub-leading.

Last ${\varphi}_{c,a}$ (${\varphi}_{c,b}$) can get corrections to the diagonal term, but they will be at most $O\left(\frac{1}{L}\right)$ while the leading term is $O\left(1\right)$.


\section{Large structure decomposition and non-linearities}
\label{appendix:xl-decomposition-nonlin}
One can write the kernel of the network when applying non-linearities in the form:
\begin{equation}
k(X,Y) = \sum_\alpha k^{L+1}_\alpha (x^{L+1},y^{L+1}) k^{L}_\alpha (\{x^{s}\}_{s=1}^L,\{y^{s}\}_{s=1}^L).
\end{equation}
for some $\left\{ k^{L+1}_\alpha, k^{L}_\alpha \right\}_\alpha$.
Since all $k^{L}_\alpha$ possess the permutation symmetry they will be diagonalized in the same basis as the symmetry operator. Suppose $\varphi_j^L \left(\{x^{s}\}_{s=1}^L\right)$ is a non-degenerate eigenfunction of the symmetry operator, we have that $\hat{K}^{L}_\alpha \varphi_j = \lambda_{\alpha,j}^L \varphi_j$ simplifying the kernel eigenvalue problem to
\begin{equation}
\hat{K} \left( \varphi^{L+1}_i  \varphi^{L}_j \right)=  \lambda_{ij} \left( \varphi^{L+1}_i  \varphi^{L}_j \right),
\end{equation}
where $\{\varphi^{L}_j\}_{j=1}^n$ are known, forming blocks of size $n$. Note that this is not a simple tensor product structure $\lambda_{ij} \neq \lambda_{i}^{L+1} \lambda_{j}^L$ as $x^{L+1}$ is not independent of $\{x^{s}\}_{s=1}^L$.

\section{WikiText-2 Symmetry Experiment Details}
\label{appendix:wikitext}
\begin{figure*}[t!]
    \centering
\includegraphics[width=.33\columnwidth]{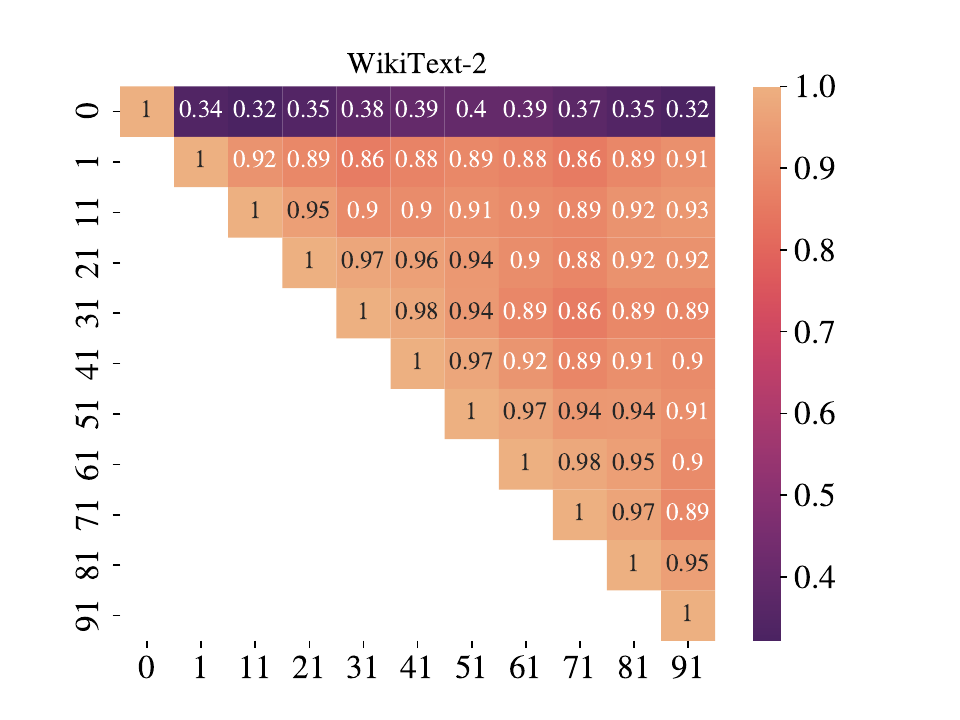}
  \vspace{-0.35cm}\hfill
\includegraphics[width=.33\columnwidth]{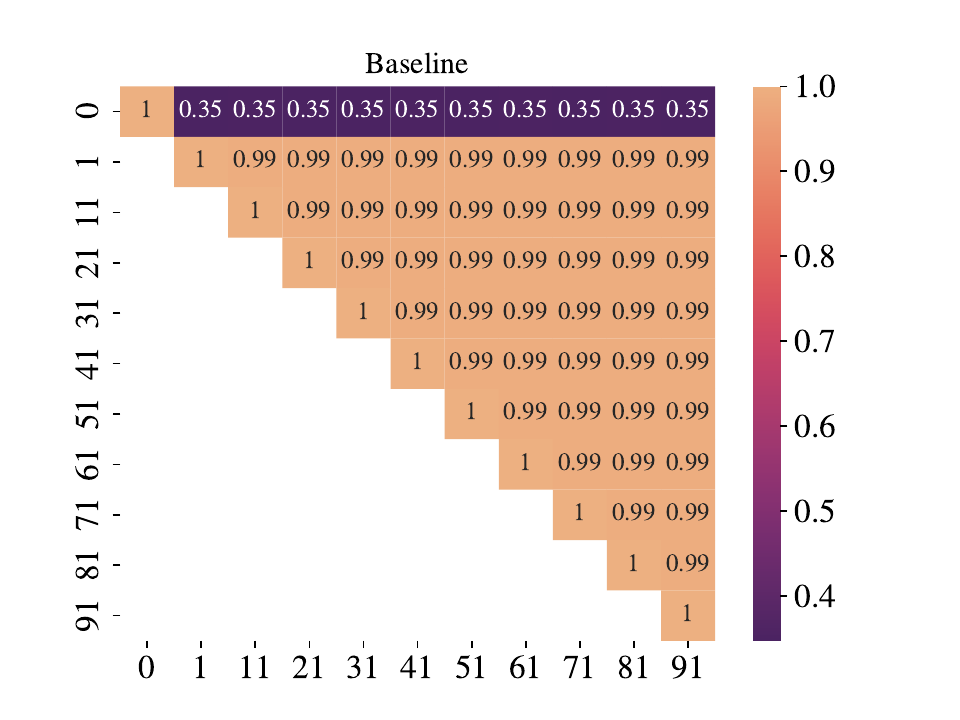}
  \vspace{-0.35cm}\hfill
\includegraphics[width=.33\columnwidth]{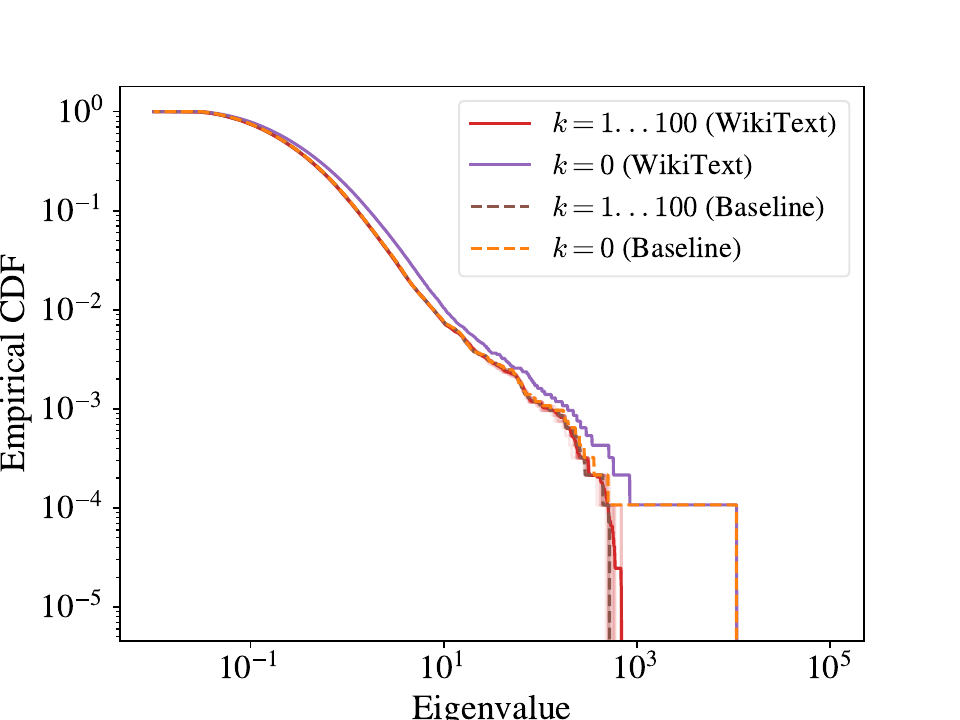}
\vspace{0.7cm} 
     \caption{ {\bf Left: (Similarity measure between $C^{kk}$ and $C^{k'k'}$ on WikiText)} the cosine similarity induced by the Frobenius inner product between the linear features of WikiText $C^{kk}$ and $C^{k'k'}$ for the $k$'s indicated on the boundary. We see all sampled $k \neq 0$ are similar to one another but different from $k=0$ as predicted by the irrep decomposition.
     {\bf Center: (Similarity measure between $C^{kk}$ and $C^{k'k'}$ on a permutation symmetric baseline dataset)} The baseline dataset is created by sampling words from WikiText with frequencies as in WikiText, but with to sequential order. The underlying distribution of baseline is therefore completely permutation symmetric in sequence space. We display the same quantity as the figure on the left, this time calculated on the baseline dataset. 
    We see all sampled $k \neq 0$ are very similar, measurably more so then the same features of WikiText. Yet, comparing the differences between the datasets to the similarity gap of $\simeq 0.6$ between $k=0$ and all $k'\neq 0$, the results for WikiText and the baseline dataset are remarkably similar, suggesting an approximate permutation symmetry.
    {\bf Right: (Comparing the spectra of $C^{kk}$ between WikiText and the permutation symmetric baseline)} The similarity between $k$s with $k\neq 0$ is again seen in the spectra. One notably difference between the baseline and WikiText is the spectrum of $C^{00}$ which differs along almost all the scale of eigenvalues, showing the principle components do capture information about sequence dependence, information that does not exist in the baseline.}
\end{figure*}
Here we give some of the details about the WikiText-2 symmetry experiment.
We started with tokenizing and trimming:
each sample was tokenized and trimmed to $L=101$ tokens. We removed any sample that was shorter than $101$ tokens, leaving us with about $10,000$ samples.

If the dataset is permutation invariant, Ideally, one would now want to perform principal component analysis (PCA) and find a set of generically $N_{\rm voc}$ different states, each with degeneracy $L-1$ for $k=1,...,L-1$ belonging to the standard irrep, and another set of generically non-degenerate $N_{\rm voc}$ different states, for $k=0$ belonging to the trivial irrep.
The PCA matrix would be
\begin{equation}
C^{a b}_{ij} := \E_{X \sim {\rm WikiText-2}} \left[ X^a_i X^b_j \right],
\end{equation}
where $a,i$ and $b,j$ can be understood as some ``flattened" super index of a $(L \cdot \nvoc) \times  (L \cdot \nvoc)$ dimensional matrix.

Moving on to Fourier space
\begin{align}
\tilde{C}^{k k'}_{ij} &:= \E_{X \sim {\rm WikiText-2}} \left[ X^a_i V^{a k} X^b_j V^{b k'} \right]; \\
V^{ak} &:= \exp(i \frac{2 \pi}{L} a k),~~ 
\begin{aligned} \vspace{-0.2cm}
    a &= 1,...,L \\
    k &= 0,...,L-1
\end{aligned} .\vspace{-0.4cm}
\end{align}
One would then expect to find a block diagonal matrix where $\tilde{C}^{kk'}_{ij}=0$ for $k \neq k'$ and $\tilde{C}^{kk}_{ij}=\tilde{C}^{k'k'}_{ij}$ for $k,k'\in \{1,...,L-1\}$.

However, since the number of samples $N<L \cdot \nvoc, \nvoc$ one cannot expect to find a block diagonal structure. Both the ranks of the matrix $\tilde{C}$ and the block $\tilde{C}^{k,k'}$ are determined by $N$, such that $\rank{\tilde{C}}=\rank{\tilde{C}^{k,k}}=N$, so the off-block-diagonal elements must not vanish to make the equality possible. A well-studied similar setting is that of the Wishart ensemble in random matrix theory~\cite{potters_first_2020,10.1093/oxfordhb/9780198744191.001.0001}.
Even with $N<L \cdot \nvoc$ we may still expect $\tilde{C}^{kk}_{ij}=\tilde{C}^{k'k'}_{ij}$ for $k,k'\in \{1,...,L-1\}$, but we would have to consider the noise due to the finite sampling.

To measure whether $\tilde{C}^{kk}_{ij}=\tilde{C}^{k'k'}_{ij}$ for $k,k'\in \{1,...,L-1\}$ we present in the main text the cosine similarity induced by the Frobenius inner product and compare the spectrum's empirical cumulative distribution function (ECDF). Here we give a more detailed plot of the similarity, where the cosine similarity value is written on top of each corresponding square

It is important to noted our results support the hypothesis WikiText-2 has an \emph{approximate} permutation symmetry, but not an exact one. We create a baseline dataset where the word frequencies are identical to those in WikiText-2, but are drawn uniformly over the sequence, i.e. the underlying distribution is exactly permutation symmetric. One can see it shows greater similarity between $k,k'\neq 0$
blocks than the one found on WikiText-2. We can also see the information contain in WikiText-2 principal components goes beyond the frequency of words, merely by the fact the results are not identical. We also compare the empirical eigenvalue CDF, showing agreement on the spectra of $\tilde{C}^{kk}$ for $k\neq0$ but a significant difference for $k=0$.

In principle, using this method, one can look at correlations up to an arbitrary order, e.g. the third-order correlator would be 
\begin{equation}
\mathcal{C}^{a b c}_{ijk} := \E_{X \sim {\rm WikiText-2}} \left[ X^a_i X^b_j X^c_k \right].
\end{equation}

\section{Full expressions of quantities in the main text}
\label{appendix:full-expressions}
Here we provide the full expressions for some of the quantities defined in the main text and in Appendix~\ref{appendix:linear_example}.

The basis chosen for the trivial representation is
\begin{equation}
\begin{aligned}\binom{{\varphi}_{c,a}}{{\varphi}_{c,b}}(X) & =\binom{\frac{1}{Z_{c,a}}}{\frac{1}{Z_{c,b}}}\binom{x_{1}^{L+1}}{1-x_{1}^{L+1}}\\
\binom{{\varphi}_{0,a}^{+}}{{\varphi}_{0,b}^{+}}(X) & =\binom{\frac{1}{Z_{0,a}^{+}}}{\frac{1}{Z_{0,b}^{+}}}\binom{x_{1}^{L+1}}{1-x_{1}^{L+1}}\frac{1}{L} &  & \left[\sum_{s=1}^{L}x^{s}-\binom{\frac{c_{a}^{{\rm odd}}+c_{a}^{{\rm even}}}{2}}{\frac{c_{b}^{{\rm odd}}+c_{b}^{{\rm even}}}{2}}\right]\\
\binom{{\varphi}_{0,a}^{-}}{{\varphi}_{0,b}^{-}}(X) & =\binom{\frac{1}{Z_{0,a}^{-}}}{\frac{1}{Z_{0,b}^{-}}}\binom{x_{1}^{L+1}}{1-x_{1}^{L+1}}\frac{1}{L} &  & \left[\binom{\alpha_{a}}{\alpha_{b}}\left(\sum_{s=1}^{L/2}x^{2s-1}-\binom{c_{a}^{{\rm odd}}}{c_{b}^{{\rm odd}}}\right)\ldots\right.\\
 &  &  & \left.\ldots-\binom{\beta_{a}}{\beta_{b}}\left(\sum_{s=1}^{L/2}x^{2s}-\binom{c_{a}^{{\rm even}}}{c_{b}^{{\rm even}}}\right)\right]
\end{aligned}
\label{eq:phi_base_def}
\end{equation}
with
\begin{equation}
\begin{aligned}
    \alpha_a =& \frac{-24 {p_a}{q_a} ({p_a}+{q_a}-2) ({p_a}+{q_a}) -12 w \left({p_a}^3+{p_a}^2 (7 {q_a}-2)+{p_a}{q_a} (7 {q_a}-8)+({q_a}-2) {q_a}^2\right)  + ...} {48 ({p_a}+{q_a}+w)} \\
    &\frac{... + 2 w^2 \left((L-16) {p_a}^2+{q_a} ((L-16) {q_a}+18)+{p_a}(18-44 {q_a})\right)
    +2 w^3 + ...}{48 ({p_a}+{q_a}+w)} \\
    &\frac{... +((L-14) {p_a}+(L-14) {q_a}+6)+(L-8) w^4
    }{48 ({p_a}+{q_a}+w)}
\end{aligned}
\end{equation}
\begin{equation}
\begin{aligned}
    \beta_a =   
    &\frac{-36 ({p_a}+{q_a}) \left(({p_a}-1) {p_a}^2+({q_a}-1) {q_a}^2\right) -18 w \left(5 {p_a}^3+{p_a}^2 (3 {q_a}-4)+{p_a}{q_a} (3 {q_a}-4)+{q_a}^2 (5 {q_a}-4)\right)+ ...}{72 ({p_a}+{q_a}+w)}\\
    &\frac{... + 6 w^2 \left({p_a}((L-12) {q_a}+10)-15 {p_a}^2+5 {q_a} (2-3 {q_a})\right)+3 w^3 ((L-18) {p_a}+(L-18) {q_a}+8)+(L-18) w^4}{72 ({p_a}+{q_a}+w)}
\end{aligned}
\end{equation}
\begin{equation}
\begin{aligned}
    \alpha_b =& 
    -\frac{-24 ({p_a}-1) ({q_a}-1) ({p_a}+{q_a}-2) ({p_a}+{q_a})+...}{48 ({p_a}+{q_a}+w-2)} \\
    &\frac{...-12 w \left({p_a}^3+{p_a}^2 (7 {q_a}-8)+{p_a}({q_a}-2) (7 {q_a}-6)+({q_a}-6) ({q_a}-2) {q_a}-4\right) + ...}{48 ({p_a}+{q_a}+w-2)}\\
    &\frac{...+2 w^2 \left(L (({p_a}-2) {p_a}+({q_a}-2) {q_a}+2)-2 \left(8 {p_a}^2+{p_a}(22 {q_a}-29)+{q_a} (8 {q_a}-29)+20\right)\right) + ...} {48 ({p_a}+{q_a}+w-2)} \\
    &\frac{...+2 w^3 (L ({p_a}+{q_a}-2)-2 (7 {p_a}+7 {q_a}-11))+(L-8) w^4}{48 ({p_a}+{q_a}+w-2)}
\end{aligned}
\end{equation}
\begin{equation}
\begin{aligned}
    \beta_b =   
    &\frac{36 ({p_a}+{q_a}-2) \left({p_a}({p_a}-1)^2+({q_a}-1)^2 q\right) +...}{72 ({p_a}+{q_a}+w-2)} \\
    &\frac{...+18 w \left(5 {p_a}^3+{p_a}^2 (3 {q_a}-14)+{p_a}({q_a} (3 {q_a}-8)+12)+{q_a} ({q_a} (5 {q_a}-14)+12)-4\right)+...}{72 ({p_a}+{q_a}+w-2)} \\
    &\frac{...+ 6 w^2 \left(L ({p_a}(-{q_a})+{p_a}+{q_a}-1)+15 {p_a}^2+4 {p_a}(3 {q_a}-8)+{q_a} (15 {q_a}-32)+22\right) +...}{72 ({p_a}+{q_a}+w-2)}\\
    &\frac{...-3 w^3 (L ({p_a}+{q_a}-2)-2 (9 {p_a}+9 {q_a}-14))-\left((L-18) w^4\right)}{72 ({p_a}+{q_a}+w-2)}
\end{aligned}
\end{equation}

\begin{equation}
\begin{aligned}
    c_{a}^{{\rm odd}} &= \frac{3 \left({p_a}^2+{q_a}^2\right)+3 w ({p_a}+{q_a})+2 w^2}{3 ({p_a}+{q_a}+w)} \\
    c_{a}^{{\rm even}} &=   
    \frac{(2 {p_a}+w) (2 {q_a}+w)}{2 ({p_a}+{q_a}+w)} \\
    c_{b}^{{\rm odd}} &= \frac{3 w ({p_a}+{q_a}-1)+3 ({p_a}-1) {p_a}+3 ({q_a}-1) {q_a}+2 w^2}{3 ({p_a}+{q_a}+w-2)} \\
    c_{b}^{{\rm even}} &= \frac{2 {p_a}(2 {q_a}+w-1)+2 {q_a} (w-1)+(w-2) w}{2 ({p_a}+{q_a}+w-2)}
\end{aligned}
\end{equation}
\begin{equation}
\begin{aligned}
\lambda^{\rm odd}_{k,a} &= \frac{1}{8 L^2} \left[ 2\left(\left(1-p_{a}\right)p_{a}^{2}+\left(1-q_{a}\right)q_{a}^{2}\right)+O(w) \right], \\
\lambda^{\rm even}_{k,a} &= \frac{1}{8 L^2} \left[ 2 p_a q_a \left(1-p_a+1-q_a\right) + O(w) \right], \\
\lambda^{\rm odd}_{k,b} &= \frac{1}{8 L^2} \left[ 2\left(p_{a}\left(1-p_{a}\right){}^{2}+q_{a}\left(1-q_{a}\right){}^{2}\right)+O(w) \right], \\
\lambda^{\rm even}_{k,b} &= \frac{1}{8 L^2} \left[ 2\left(1-p_{a}\right)\left(1-q_{a}\right)\left(p_{a}+q_{a}\right)+O(w) \right]
\end{aligned}
\label{eq:appndx_k_eigvals}
\end{equation}
To leading order in $\frac{1}{L},w$, the spanning coefficients of the learnable target are given by 
\begin{equation}
\begin{aligned}
g_{k,*}^{{\rm odd}} & =0, &g_{k,*}^{{\rm even}} &= 0\\
g_{*}^{+} & =\frac{Lw^{2}\eta_{*}^{+}}{\sqrt{L^{2}w^{2}\rho_{*}^{+}+L\sigma_{*}^{+}}},
& g_{*}^{-} & =\frac{Lw^{2}\eta_{*}^{-}+\nu_{*}^{-}}{\sqrt{L^{2}w^{4}\rho_{*}^{-}+Lw^{2}\sigma_{*}^{-}+\xi_{*}^{-}}},\\
g_{c,a} & =\frac{p_{a}q_{a}}{\sqrt{\frac{p_{a}+q_{a}}{2}}},
&g_{c,b} & =\frac{q_{a}+p_{a}-2p_{a}q_{a}}{\sqrt{2\left(1-p_{a}+1-q_{a}\right)}};
\end{aligned}
\label{eq:learning-target-coeff}
\end{equation}
with
\begin{equation}
\begin{aligned}
    \eta^{+}_{0,a} &= 2 \left({p_a}^2+{q_a}^2\right) \\
    \rho^{+}_{0,a} &= 48 ({p_a}+{q_a})^3 \\
    \sigma^{+}_{0,a} &= -576 ({p_a}+{q_a})^3 (({p_a}-1) {p_a}+({q_a}-1) {q_a})
\end{aligned}
\end{equation}

\begin{equation}
\begin{aligned}
    \eta^{+}_{0,b} &= 2 ({p_a}-2) {p_a}+2 ({q_a}-2) {q_a}+4 \\
    \rho^{+}_{0,b} &= 2 ({p_a}-2) {p_a}+2 ({q_a}-2) {q_a}+4 \\
    \sigma^{+}_{0,b} &= 576 ({p_a}+{q_a}-2)^3 (({p_a}-1) {p_a}+({q_a}-1) {q_a})
\end{aligned}
\end{equation}

\begin{equation}
\begin{aligned}
    \eta^{-}_{0,a} &= -72 {p_a}{q_a} ({p_a}-{q_a})^2 ({p_a}+{q_a})\\
    \nu^{-}_{0,a} &= 864 {p_a}{q_a} ({p_a}-{q_a})^2 ({p_a}+{q_a}) (({p_a}-1) {p_a}+({q_a}-1) {q_a}) \\
    \rho^{-}_{0,a} &= 10368 {p_a}{q_a} ({p_a}-{q_a})^2 ({p_a}+{q_a})^3 \\
    \sigma^{-}_{0,a} &= -248832 {p_a}{q_a} ({p_a}-{q_a})^2 ({p_a}+{q_a})^3 (({p_a}-1) {p_a}+({q_a}-1) {q_a}) \\
    \xi^{-}_{0,a} &= 1492992 {p_a}{q_a} ({p_a}-{q_a})^2 ({p_a}+{q_a})^3 (({p_a}-1) {p_a}+({q_a}-1) {q_a})^2
\end{aligned}
\end{equation}

\begin{equation}
\begin{aligned}
    \eta^{-}_{0,a} &= -72 ({p_a}-1) ({q_a}-1) ({p_a}-{q_a})^2 ({p_a}+{q_a}-2)\\
    \nu^{-}_{0,a} &= 864 ({p_a}-1) ({q_a}-1) ({p_a}-{q_a})^2 ({p_a}+{q_a}-2) (({p_a}-1) {p_a}+({q_a}-1) {q_a}) \\
    \rho^{-}_{0,a} &= -10368 ({p_a}-1) ({q_a}-1) ({p_a}-{q_a})^2 ({p_a}+{q_a}-2)^3 \\
    \sigma^{-}_{0,a} &= 248832 ({p_a}-1) ({q_a}-1) ({p_a}-{q_a})^2 ({p_a}+{q_a}-2)^3 (({p_a}-1) {p_a}+({q_a}-1) {q_a}) \\
    \xi^{-}_{0,a} &= -1492992 ({p_a}-1) ({q_a}-1) ({p_a}-{q_a})^2 ({p_a}+{q_a}-2)^3 (({p_a}-1) {p_a}+({q_a}-1) {q_a})^2
\end{aligned}
\end{equation}
\end{document}